\documentclass[11pt]{article}

\usepackage[margin=1in]{geometry}
\usepackage[authoryear,round]{natbib}

\usepackage{microtype}
\usepackage{graphicx}
\usepackage{subfigure}
\usepackage{booktabs} 

\usepackage{url} 
\usepackage{booktabs} 
\usepackage{amsfonts} 
\usepackage{nicefrac} 
\usepackage{microtype}
\usepackage{graphicx}
\usepackage{wrapfig}
\usepackage{booktabs} 
\usepackage{def}
\usepackage{xspace}
\usepackage{algorithm}
\usepackage{algorithmic}
\usepackage{amssymb}
\usepackage{enumitem}
\usepackage{comment}
\usepackage{bm}
\usepackage{bbm}
\usepackage{breqn}
\usepackage{authblk}
\usepackage[authoryear,round]{natbib}
\usepackage{mdframed} 
\usepackage{amsmath,amsthm}

\newcommand{\wo}[1]{\widetilde{\mathcal{O}}\left( #1 \right)}
\newcommand{\bo}[1]{\mathcal{O}\left( #1 \right)}
\usepackage{multirow}
\newcommand{\E}{\mathbf{E}}
\definecolor{darkgreen}{RGB}{0, 100, 0}

\usepackage[colorlinks=true, linkcolor=blue, citecolor=blue]{hyperref}
\usepackage[capitalize,noabbrev]{cleveref}

\usepackage[textsize=tiny]{todonotes}

\title{
Towards Global Optimality for Practical Average Reward Reinforcement Learning without Mixing Time Oracles
}

\usepackage{authblk}
\author[1]{Bhrij Patel} 
\author[2]{Wesley A. Suttle} 
\author[3]{Alec Koppel}
\author[4]{Vaneet Aggarwal} 
\author[2]{Brian M. Sadler}
\author[5]{Amrit Singh Bedi}
\author[1]{Dinesh Manocha}
\affil[1]{University of Maryland, College Park}
\affil[2]{U.S. Army Research Laboratory, MD, USA}
\affil[3]{JP Morgan AI Research, NYC}
\affil[4]{Purdue University}
\affil[5]{University of Central Florida}

\renewcommand{\eqref}[1]{(\ref{#1})}
\begin{document}
\date{}
\maketitle

\begin{abstract}
    In the context of average-reward reinforcement learning, the requirement for oracle knowledge of the mixing time, a measure of the duration a Markov chain under a fixed policy needs to achieve its stationary distribution, poses a significant challenge for the global convergence of policy gradient methods. This requirement is particularly problematic due to the difficulty and expense of estimating mixing time in environments with large state spaces, leading to the necessity of impractically long trajectories for effective gradient estimation in practical applications.
To address this limitation, we consider the Multi-level Actor-Critic (MAC) framework, which incorporates a Multi-level Monte-Carlo (MLMC) gradient estimator. With our approach, we effectively alleviate the dependency on mixing time knowledge, a first for average-reward MDPs global convergence. Furthermore, our approach exhibits the tightest available dependence of $\bo{\sqrt{\tau_{mix}}}$ known from prior work. With a 2D grid world goal-reaching navigation experiment, we demonstrate that MAC outperforms the existing state-of-the-art policy gradient-based method for average reward settings. 
\end{abstract}

\section{Introduction}


In reinforcement learning (RL) problems, temporal dependence of data breaks the independent and identically distributed (i.i.d.) assumption commonly encountered in machine learning analyses, rendering the theoretical analysis of RL methods challenging. In discounted RL, the impact of temporal dependence is typically mitigated, as the effect of the discount factor renders the stationary behavior of the induced Markov chains irrelevant. On the other hand, in average-reward RL, stationary behavior under induced policies is of fundamental importance. In particular, understanding the effect of mixing time, a measure of how long a Markov chain takes to approach stationarity, is critical to the development and analysis of average-reward RL methods \citep{suttle2023beyond, riemer2021continual}. Given the usefulness of the average-reward regime in applications such as robotic locomotion \citep{zhang2021}, traffic engineering \citep{geng2020multi}, and healthcare \citep{ling2023cooperating}, improving our understanding of the issues inherent in average-reward RL is increasingly important.

Key to theoretically understanding a learning method is characterizing its convergence behavior. For a method to be considered sound, we should ideally be able to prove that, under suitable conditions, it converges to a globally optimal solution while remaining sample-efficient. Convergence to global optimality of policy gradient (PG) methods \citep{sutton2018reinforcement}, a subset of RL methods well-suited to problems with large and complex state and action spaces, has been extensively studied in the discounted setting \citep{bhandari2019global, agarwal2020optimality, mei2020global, liu2020improved, bedi2022hidden}. Due to gradient estimation issues arising from the mixing time dependence inherent in the average-reward setting, however, the problem of obtaining global optimality results for average-reward PG methods remained open until recently.

In \citet{bai2023regret}, the Parameterized Policy Gradient with Advantage Estimation (PPGAE) method was proposed and shown to converge to a globally optimal solution in average-reward problems under suitable conditions. However, the implementation of the PPGAE algorithm relies on oracle knowledge of mixing times, which are typically unknown and costly to estimate, and requires extremely long trajectory lengths at each gradient estimation step. These drawbacks render PPGAE costly and sample-inefficient, leaving open the problem of developing a practical average-reward PG method that enjoys global optimality guarantees.
Recently, \citet{suttle2023beyond} proposed and analyzed the Multi-level Actor-critic (MAC) algorithm, an average-reward PG method that enjoys state-of-the-art sample complexity, avoids oracle knowledge of mixing times, and leverages a multi-level Monte Carlo (MLMC) gradient estimation scheme to keep trajectory lengths manageable. Despite these advantages, convergence to global optimality has not yet been provided for the MAC algorithm.

In this paper, we establish for the first time convergence to global optimality of an average-reward PG algorithm that does not require oracle knowledge of mixing times, uses practical trajectory lengths, and enjoys the best known dependence of convergence rate on mixing time.
To achieve this, we extend the convergence analysis of \cite{bai2023regret} to the MAC algorithm of \cite{suttle2023beyond}, closing an outstanding gap in the theory of average-reward PG methods. In addition, we provide goal-reaching navigation  results illustrating the superiority of MAC over PPGAE, lending further support to our theoretical contributions.
We summarize our contributions as follows:
\begin{itemize}
    \item We prove convergence of MAC to global optimality in the infinite horizon average reward setting.
    \item Despite lack of mixing time knowledge, we achieve a tighter mixing time dependence, $\bo{\sqrt{\tau_{mix}}}$, than previous average-reward PG algorithms.
    \item We highlight the practical feasibility of the MAC compared with PPGAE by empirically comparing their sample complexities in a 2D gridworld goal-reaching navigation task where MAC achieves a higher reward.
\end{itemize}
%
%
The paper is organized as follows: in the next section we given an overview of related works in policy gradient algorithms and mixing time; Section \ref{section:prob_formulation} describes general problem formulation for average reward policy gradient algorithms and then specifically details the MAC algorithm along with PPGAE from \cite{bai2023regret}; Section \ref{section:global_convergence} presents our global convergence guarantees of MAC and provides a discussion comparing the practicality of MAC to PPGAE. We also provide a 2D gridworld goal-reaching navigation experiment where MAC achieves a higher reward than PPGAE; we the end paper with conclusion and discussion on future work. 

\begin{table*}[t]
    \centering
    \caption{ This table compares the different policy gradient algorithms for the average reward setting and their global convergence rates. Out of all the papers with an explicit dependence on mixing time, MAC from \cite{suttle2023beyond}, which we analyze in this paper, has the tightest dependence.\vspace{2mm}}
    \resizebox{\textwidth}{!}{
        \begin{tabular}{|c||c|c|c|c|c|}
    	\hline
    	Algorithm & Reference  & Mixing Time Known & Mixing Time Dependence & Convergence Rate & Parameterization \\

        \hline
            FOPO & \cite{wei2021learning} &  Yes & N/A &  $\wo{T^{-\frac{1}{2}}}$ & Linear \\
    	\hline
            OLSVI.FH & \cite{wei2021learning} &  Yes & N/A &  $\wo{T^{-\frac{1}{4}}}$ & Linear \\
    	\hline
             MDP-EXP2 & \cite{wei2021learning} &  Yes & $\wo{\sqrt{\tau_{mix}^3}}$ &  $\wo{T^{-\frac{1}{2}}}$ & Linear \\
    	\hline
            PPGAE & \cite{bai2023regret} & Yes & $\wo{\tau_{mix}^2}$ &  $\wo{T^{-\frac{1}{4}}}$ & General \\
    	\hline
           MAC \textbf{(This work)} & \cite{suttle2023beyond}  & No & $\wo{\sqrt{\tau_{mix}}}$ &  $\wo{T^{-\frac{1}{4}}}$ & General \\
    	\hline
        \end{tabular}
    }

    \label{table2}
\end{table*}

\section{Related Works} \label{Related_Works}



In this section, we provide a brief overview of the related works for global optimality of policy gradient algorithms and for mixing time. 



\textbf{Policy Gradient.} Its global optimality has been shown to exist for softmax \citep{mei2020global} and tabular \citep{bhandari2019global, agarwal2020optimality} parameterizations. For the discounted reward setting \citet{liu2020improved} provided a general framework for global optimality for PG and natural PG methods for the discounted reward setting. Recently, \citet{bai2023regret} adapted this framework for the average reward infinite horizon MDP with general policy parameterization. We wish to apply this framework for the MAC algorithm to introduce a global optimality analysis in the average reward setting with no oracle knowledge of mixing time.

\textbf{Mixing Time.} Previous works have emphasized the challenges and infeasibility of estimating mixing time \citep{hsu2015, wolfermixing2020} in complex environments. Recently, \citet{patel2023ada} used policy entropy as a proxy variable for mixing time for an adaptive trajectory length scheme. Previous works that assume oracle knowledge of mixing time such as \citet{bai2023regret, duchi2012, bresler2020} are limited in practicality. In \citet{suttle2023beyond}, they relaxed this assumption with their proposed Multi-level Actor-Critic (MAC) while still recovering SOTA convergence. In this paper, we aim to establish the global convergence of MAC and highlight its tighter dependence on mixing time despite no oracle knowledge.

\section{Problem Formulation}\label{section:prob_formulation}

    

\subsection{Average Reward Policy Optimization}

Reinforcement learning problem with the average reward criterion may be formalized as a Markov decision process $\mathcal{M} := (\mathcal{S}, \mathcal{A}, \mathbb P, r)$. In this tuple expression, $\mathcal{S}$ and $\mathcal{A}$ are the finite state and action spaces, respectively; $\mathbb P(\cdot~|~s,a)$ maps the current state $s \in \mathcal{S}$ and action $a \in \mathcal{A}$ to the conditional probability distribution of next state $s^\prime \in \mathcal{S}$, and $r: \mathcal{S} \times \mathcal{A} \to [0, r_{\max}]$ is a bounded reward function. An agent, starting from state $s_t\in\mathcal{S}$, selects actions $a_t\in\mathcal{A}$ which causes a transition to a new state $s_t^\prime\sim \mathbb P(\cdot~|~s_t,a_t)$ and the environment reveals a reward $r(s_t,a_t)$. Actions may be selected according to a policy $\pi (\cdot ~|~ s)$, which is a  distribution over action space $\mathcal{A}$ given current state $s$. 

We aim to maximize the long-term average reward $J(\pi):=\lim_{T\rightarrow \infty}\mathbb{E}\left[\frac{1}{T}\sum_{t=0}^Tr(s_t,a_t)\right]$ by finding the optimal policy $\pi$. The policy is parameterized by  vector $\theta \in \mathbb R^q$, where $q$ denotes the parameter dimension, and the policy dependence on $\theta$ is indicated via the notation $\pi_{\theta}$.  Parameterization in practice can vary widely, from neural networks to tabular representations. This work, like in \cite{bai2023regret}, aims to provide a global convergence guarantee with no assumption on policy parameterization. With this notation, we can formalize the objective as solving the following maximization problem: 

%

\begin{align}\label{eq:policy_opt}
    \max_{\theta} J(\pi_{\theta}):=\lim_{T\rightarrow \infty}\mathbb{E}_{s_{t+1}\sim \mathbb{P}(\cdot|s_t,a_t), a_t\sim\pi_{\theta}(\cdot|s_t)}\left[R_T\right],
\end{align}
where $R_T:=\frac{1}{T}\sum_{t=0}^Tr(s_t,a_t)$.  Observe that, in general, \eqref{eq:policy_opt} is non-convex with respect to $\theta$, which is the critical challenge of applying first-order iterations to solve this problem -- see \cite{zhang2020global,agarwal2020optimality}. Further define the stationary distribution of induced by a parameterized policy as,
\begin{align}
    d^{\pi_{\theta}}(s) = \lim_{T\rightarrow \infty}\dfrac{1}{T}\left[\sum_{t=0}^{T-1} \mathrm{Pr}(s_t=s|s_0\sim \rho, \pi_{\theta})\right].
\end{align}

As we will later discuss, the induced $d^{\pi_{\theta}}$ for a given $\theta$ is unique and therefore, is agnostic to initial distribution $\rho$, due to ergodic assumption. We are now able to express the average reward with respect $d^{\pi_{\theta}}$ induced by a parameterized policy as $J(\pi_{\theta})= \mathbb{E}_{s\sim d^{\pi_{\theta}}, a\sim {\pi_{\theta}}}[r(s,a)]$.  
%
This equation thus reveals the explicit dependence our optimization problem has on $d^{\pi_{\theta}}$. It is assumed that the data sampled comes from the unique stationary distribution. Any samples that are generated before the induced Markov Chain reaches a stationary distribution are known as burn-out samples and lead to noisy gradients. Thus, knowing the time it takes an induced Markov Chain has reached its stationary distribution is a crucial element for policy gradient algorithms. The quantity is known as \textit{mixing time} $\tau_{mix}$ as is defined as,

\begin{definition}[$\epsilon$-Mixing Time] Let $d^{\pi_{\theta}}$ denote the stationary distribution of the Markov chain induced by $\pi_{\theta}$. We first define the metric,
\begin{equation}
m(t; \theta) := \sup_{s\in \mathcal{S}} \|P^t(\cdot | s)- d^{\pi_{\theta}}(\cdot) \|_{TV},
\end{equation}
where $\|\cdot\|_{TV}$ is the total variation distance.

The $\epsilon$-mixing time of a Markov chain is defined as%
\begin{equation}\label{mixing time}
\tau_{mix}^\theta(\epsilon) := \inf \{t: m(t; \theta) \leq \epsilon\},
\end{equation}
Typical in PG analysis, mixing time is defined as $\tau_{mix}^\theta:=\tau_{mix}^\theta(1/4)$ as it provides a result in that $m(l\tau^{\theta}_{mix}; \theta) \leq 2^{-l}\}$ \cite{dorfman2022}. We further define $\tau_{mix} = \max_{t \in [T]} \tau_{mix}^{\theta_t}$.
\end{definition}

In prior works such as \cite{duchi2012, bresler2020, bai2023regret, wei2021learning}, $\tau_{mix}^\theta$ is assumed to be known. However for complex environments this is rarely the case \citep{hsu2015, wolfermixing2020}. In Section \ref{section:mac} we discuss the Multi-level Actor-Critic algorithm and how it relaxes this assumption. We will first, however, briefly introduce the elements of a vanilla actor-critic in Section \ref{section:ac}.


To do so, we define  the action-value ($Q$) function as 
\begin{align}\label{q_function}
    Q^{\pi_{\theta}}(s,a) =& \mathbb{E}\Bigg[\sum_{t=0}^{\infty}\mathbb [r(s_t,a_t) - J(\pi_{\theta})]\Bigg],
    \end{align}
    such that $s_0 = s, a_0 = a$, and action $a\sim \pi_{\theta}$. We can then further write the state value function as
    \begin{align}\label{v_function}
     V^{\pi_{\theta}}(s) =& \mathbb E_{a \sim \pi_{\theta}(\cdot|s)}[Q^{\pi_{\theta}}(s,a)].
\end{align}
Using the Bellman's Equation, we can write, from \eqref{q_function} and \eqref{v_function}, the value of a state $s$, in terms of another as \citep{puterman2014markov}
\begin{align}\label{bellman}
V^{\pi_{\theta}}(s)
&= \mathbb E[r(s,a) - J(\pi_{\theta}) + V^{\pi_{\theta}}(s')],
\end{align}
where the expectation is over $a \sim \pi_{\theta}(\cdot|s), s'\sim \mathbb P(\cdot |a, s)$. 

We also define the advantage term as follows, 
\begin{align}
    A^{\pi_{\theta}}(s, a) \triangleq Q^{\pi_{\theta}}(s, a) - V^{\pi_{\theta}}(s).
\end{align}

With this, we can now present the well-known policy gradient theorem established by \cite{sutton1999policy},
\begin{lemma}
    \label{lemma_grad_compute}
    The gradient of the long-term average reward can be expressed as follows.
    \begin{align}
        \!\!\!\!\nabla_{\theta} J(\theta)\!=\!\mathbf{E}_{s\sim d^{\pi_{\theta}},a\sim\pi_{\theta}(\cdot|s)}\!\bigg[ \!A^{\pi_{\theta}}(s,a)\nabla_{\theta}\log\pi_{\theta}(a|s)\!\bigg].
    \end{align}
\end{lemma}

PG algorithms maximize the average reward by updating $\theta$ using a gradient ascent step with stepsize $\alpha$\footnote{With slight abuse of notation, we use $t$ as an index for sample number and gradient update to align with prior work in mixing time \citep{suttle2023beyond, dorfman2022}.},

\begin{align}
    \theta_{t+1}= \theta_t + \alpha_t \nabla_{\theta} J(\pi_{\theta_t}),
\end{align}

Normally, average reward policy gradient algorithms estimate the advantage function with a simple average of all reward values observed during a trajectory. In \cite{bai2023regret} they propose an advantage estimation algorithm that estimates $Q$ and $V$ values from the trajectories sampled by splitting them into sub-trajectories. However, they propose dividing the trajectory into sub-trajectories, where the length of the length of trajectory and sub-trajectories are functions of mixing time. Thus, we aim to provide a global convergence guarantee of  an algorithm that has no requirement of knowing mixing time. This goal motivates us to select MAC which has no such requirement but lacks global convergence analysis. We give an high-level overview of the algorithm from \cite{bai2023regret} in the following subsection. 

\subsection{Parameterized Policy Gradient with Advantage Estimation}\label{section:ppgae}
Recently, \cite{bai2023regret} proposed a PG algorithm, Parameterized Policy Gradient with Advantage Estimation (PPGAE), and derived its global convergence gaurantee in the average reward setting for general policy parameterization. PPGAE defines $K$ as the number of epochs, $H$ as the length of one epoch, and $T$ as the sample budget of the entire training process. Therefore, we can express them in terms of each with $K = T/H$. To formulate the policy gradient update at the end of the epoch, the advantage value of each sample collected in an epoch is then estimated based on the reward values observed in the samples. A more detailed description of the PPGAE algorithm can be found in the Appendix. In PPGAE, $H = 16 \tau_{hit} \tau_{mix} \sqrt{T} (\log T)^2$, where $\tau_{hit}$, \textit{hitting time} is defined as below:	%
\begin{equation}\label{hitting time}
\tau_{hit} := \max_{\theta} \max_{s \in \mathcal{S}} \frac{1}{d^{\pi_{\theta}}(s)}.
\end{equation}
Intuitively, it is the amount of time to reach all states in the state space. If the induced Markov Chain is ergodic, then $\tau_{hit}$ is finite because each state in $\mathcal{S}$ has a non-zero chance of being reached. Similar to mixing time, hitting time is also defined by the stationary distribution of a given policy. Thus, its estimation also suffers from the same difficulties as mixing time estimation.


The algorithm relies on knowing the mixing time and hitting time to calculate $H$, restricting its use case to simple environments with small state spaces where they can be feasibly be estimated. This requirement becomes more impractical as the state space or environment complexity increases \citep{hsu2015, wolfermixing2020}. Furthermore, even if mixing time and hitting time are known, by definition of epoch length, $H$, the minimum sample budget, $T$, required for the number of episodes, $K$, to be at least one is practically infeasible as we will explain in Section \ref{section:global_convergence}. In this work, we utilize a variant of the actor-critic (AC) algorithm that is able to estimate the advantage with a trajectory length scheme that has no dependence on mixing time, hitting time, and total sample budget as we will explain in more detail in the following sections.


%


\subsection{Actor-Critic Algorithm}\label{section:ac}

The AC algorithm alternates between a actor and a critic. The actor function is the parameterized policy $\pi_{\theta}$ and the critic function estimates the value function $V^{\pi_{\theta}}(s)$. We can rewrite the policy gradient theorem in terms of the \textit{temporal difference} (TD), $\delta^{\pi_{\theta}}$

    \begin{align}
        &\nabla_{\theta} J(\theta)=\mathbf{E}_{s\sim d^{\pi_{\theta}},a\sim\pi_{\theta}(\cdot|s), s'\sim p(\cdot | s, a)}\bigg[ \delta^{\pi_{\theta}}\nabla_{\theta}\log\pi_{\theta}(a|s)\bigg],
    \end{align}
where $\delta^{\pi_{\theta}} = r(s,a) - J(\theta) + V^{\pi_{\theta}}(s') - V^{\pi_{\theta}}(s)$. We can see that the TD is an estimation of the advantage term. Actor-critic algorithms provide a greater stability that alternative PG algoirhtms, such as REINFORCE \citep{williams1992simple} and PPGAE that estimate the advantage with a sum of observed rewards. The stability comes from the learned value estimator, the critic function, being used as a baseline to reduce variance in the gradient estimation.

Because the scope of this work focuses on the global convergence for the actor, we assume that the critic function is the inner product between a given feature map $\phi(s): \mathcal{S}\to \mathbb R^m$ and a weight vector $\omega \in \mathbb R^m$. This assumption allows the critic optimization problem we describe below to be strongly convex.

 We denote the critic estimation for $V^{\pi_\theta}(s)$ as $V_{\omega}(s) \coloneqq \langle\phi(s),\omega \rangle$ and assume that $\lVert \phi(s) \rVert \leq 1$ for all $s \in \mathcal{S}$.  The critic aims to minimize the error below,
\begin{align}\label{value_function_problem}
    \min_{\omega\in\Omega}  \sum_{s \in S} d^{\pi_{\theta}}(s) (V^{\pi_{\theta}}(s) - V_{\omega}(s))^2.
\end{align}
%
By weighting the summation by $d^{\pi_{\theta}}(s)$, it is more imperative to find an $\omega$ that accurately estimates of $V$ value at states where the agent has a higher probability of being in the long-run.
The gradient update for $\omega$ is given as 
\begin{align}\label{Critic_update_00}
    \omega_{t+1}=& \Pi_{\Omega}\big[\omega_t-\beta_t \big(r(s_t,a_t)-J(\pi_{\theta_t})+ \langle\phi(s_{t+1}),\omega_t\rangle
    \nonumber
    \\
    &\hspace{25mm}-\langle\phi(s_t),\omega_t\rangle\big)\phi(s_t)\big],
\end{align}
where $\beta_t$ is the critic learning rate. Because the critic update in \eqref{Critic_update_00} relies on $J(\pi_{\theta_t})$, which we do not have access to, we can substitute with a recursive estimate for the average reward as $\eta_{t+1}= \eta_t-\gamma_t (\eta_t-r(s_t,a_t))$. We now write the AC updates as
\begin{align}\label{Critic_update_0}
	    \eta_{t+1}=& \eta_t-\gamma_t \cdot f_t, && \text{(reward tracking)}
	    \nonumber
	    \\
    \omega_{t+1}=& \Pi_{\Omega}\big[\omega_t-\beta_t\cdot g_t\big],&& \text{(critic update)}
    \nonumber\\
        \theta_{t+1}=& \theta_t + \alpha_t \cdot h_t,&& \text{(actor update)}
\end{align}
where we have 
\begin{align}\label{stochastic_gradients}
 f_t=& \eta_t-r(s_t,a_t),
 \nonumber
 \\
    g_t=&\big(r(s_t,a_t)-\eta_t+ \langle\phi(s_{t+1})-\phi(s_{t}),\omega_t\rangle\big)\phi(s_t),
    \nonumber
    \\
    h_t=& \delta^{\pi_{\theta_t}}\cdot \nabla_\theta \log \pi_{\theta_t} (a_t|s_t),
    \nonumber
    \\
    \delta^{\pi_{\theta_t}}=& r(s_t,a_t)-\eta_t+ \langle\phi(s_{t+1})-\phi(s_{t}),\omega_t\rangle.
\end{align}

As the critic and reward tracking are vital to the average reward AC, we incorporate the critic and average reward tracking errors in our global convergence analysis of the actor. One drawback of most vanilla AC algorithms is their assumption on the decay rates of mixing time. Under ergodicity, Markov Chains induced by $\pi_{\theta}$ reach their respective $d^{\pi_{\theta}}$ exponentially fast. That is for some $\rho \in \left[0,1\right], m(t;\theta) \leq \bo{\rho^t}$. However most vanilla AC analysis assumes there exists some $\rho$ such that for all $\theta, m(t;\theta) \leq \bo{\rho^t}$. This assumption sets an upper limit on how slow mixing an environment can be for the algorithm to handle. In the following section, we explain the AC variant, MAC, that will provide with the tighter dependence on mixing time despite no oracle knowledge of it and no limit on $\rho$. 

\subsection{Multi-level Actor-Critic}\label{section:mac}

Recent work has developed the Multi-Level Actor-Critic (MAC) \citep{suttle2023beyond} that relies upon a Multi-level Monte-Carlo (MLMC) gradient estimator for the actor, critic, and reward tracking. Let $J_t \sim \text{Geom(}{1}/{2}\text{)}$ and we collect the trajectory $\mathcal{T}_t:=\{s_t^i,a_t^i,r_t^i,s_t^{i+1}\}_{i=1}^{2^{J_t}}$ with policy $\pi_{\theta_t}$. Then the MLMC policy gradient estimator is given by the following conditional:

\begin{equation}\label{eq:MLMC_Gradient}
h_t^{MLMC} = h_t^0 +
\begin{cases}
2^{J_t}(h_t^{J_t} - h_t^{J_t - 1}),& \text{if } 2^{J_t}\leq T_{\max}\\
0,              & \text{otherwise}
\end{cases}             
\end{equation}
with $h_t^j = \frac{1}{2^j}\sum_{i=1}^{2^j}h (\theta_t;s_t^i,a_t^i)$ and where $T_{\max} \geq 2$. We note that the same formula is used for MLMC gradient estimators of the critic, $g_t^{MLMC}$, and reward tracker, $f_t^{MLMC}$.

As we will see from Lemma \ref{lemma:31ss}, the advantage of the MLMC estimator is that we get the same bias as averaging $T$ gradients with $\wo{1}$ samples.
Drawing from a geometric distribution has no dependence on knowing what the mixing time is, thus allowing us to drop the oracle knowledge assumption previously used in works such as \cite{bai2023regret}.
To reduce the variance introduced by the MLMC estimator \citep{dorfman2022, suttle2023beyond} utilized the adpative stepsize scheme, Adagrad \citep{duchi2011, levy2017online}.

\section{Global Convergence Analysis} 
In this section we provide theoretical guarantee of the global convergence of MAC from \cite{suttle2023beyond}. 

\label{section:global_convergence}
\subsection{Preliminaries} \label{section:prelim}
In this section, we provide assumptions and lemmas that will support our global convergence analysis in the next section.
\begin{assumption}
    \label{assumption:erogidcity}
    For all $\theta$, the parameterized MDP $\mathcal{M_{\theta}}$ induces an ergodic Markov Chain.
\end{assumption}

Assumption \ref{assumption:erogidcity} is typical in many works such as \cite{suttle2023beyond, pesquerel2022imed,gong2020duality, bai2023regret}. As previously mentions, it ensures all states are reachable, and also importantly, guarantees a unique stationary distribution $d^{\pi_{\theta}}$ of any induced Markov Chain. 

Because we parameterize the critic as a linear function approximator, for a fixed policy parameter $\theta$, the temporal difference will converge to the minimum of the mean squared projected Bellman error (MSPBE) as discussed in \cite{sutton2018reinforcement}.  
\begin{definition}\label{assum:mspbe}
Denoting $\omega^*(\theta)$ as the fixed point for a given $\theta$, and for a given feature mapping $\phi$ for the critic, we define the worst-case approximation error to be
\begin{align}
    \mathcal{E}^{critic}_{app} = \sup_{\theta} \sqrt{ \mathbf{E}_{s \sim \mu_{\theta}} \left[ \phi(s)^T \omega^*(\theta) - V^{\pi_{\theta}}(s) \right]^2},
\end{align}
\end{definition}

which we assume to be finite. With a well-designed feature map, $\mathcal{E}^{critic}_{app}$ will be small or even 0. We will later see that by assuming $\mathcal{E}^{critic}_{app} = 0$ we recover the $\wo{T^{-\frac{1}{4}}}$ dependence as in \cite{bai2023regret}. 




\begin{assumption} \label{assum:policy_conditions}
    Let $\{ \pi_{\theta} \}_{\theta \in \mathbb{R}^d}$ denote our parameterized policy class. There exist $B, R, L > 0$ such that
    \begin{enumerate}
        \setlength\itemsep{0em}
        \item $\norm{ \nabla \log \pi_{\theta}(a | s) } \leq B$, for all $\theta \in \mathbb{R}^d$,
        \item $\norm{ \nabla \log \pi_{\theta}(a | s) - \nabla \log \pi_{\theta'}(a | s) } \leq R \norm{ \theta - \theta' }$, for all $\theta, \theta' \in \mathbb{R}^d$,
        \item $| \pi_{\theta}(a | s) - \pi_{\theta'}(a | s) | \leq L \norm{\theta - \theta'}$, for all $\theta, \theta' \in \mathbb{R}^d$.
    \end{enumerate}
\end{assumption}

Assumption \ref{assum:policy_conditions} establishes regularization conditions for the policy gradient ascent and has been utilized in prior work such as \cite{suttle2023beyond, papini2018stochastic, kumar2019sample, zhang2020global, xu2020improved, bai2023regret}. This assumption will be vital for presenting our modified general framework for non-constant stepsize in Lemma \ref{lem_framework} as $B, R, L$ will appear in our bound for the difference between optimal reward and the average cumulative reward.

\begin{assumption}
    \label{assump_transfer_error}
    Define the transferred policy function approximation error
    \begin{align}
        \label{eq:transfer_error}
        \begin{split}
            L_{d_\rho^{\pi^*},\pi^*}(h^*_\theta,\theta ) =\mathbf{E}_{s\sim d_\rho^{\pi^*}}\mathbf{E}_{a\sim\pi^*(\cdot\vert s)}\bigg[
            \bigg(\nabla_\theta\log\pi_{\theta}(a\vert s)\cdot h^*_{\theta}-A^{\pi_\theta}(s,a)\bigg)^2\bigg],
	\end{split}
    \end{align}
    where $\pi^*$ is the optimal policy and $h^*_{\theta}$ is given as
    \begin{align}
        \label{eq:NPG_direction}
	\begin{split}
            h^*_{\theta}=\arg\min_{h\in\mathbb{R}^{\mathrm{d}}}~\mathbf{E}_{s\sim d_\rho^{\pi_{\theta}}}\mathbf{E}_{a\sim\pi_{\theta}(\cdot\vert s)}\bigg[
            \bigg(\nabla_\theta\log\pi_{\theta}(a\vert s)\cdot h-A^{\pi_{\theta}}(s,a)\bigg)^2\bigg].
	\end{split}
    \end{align}
    We assume that the error satisfies $L_{d_{\rho}^{\pi^*},\pi^*}(h^*_{\theta},\theta)\leq \mathcal{E}^{actor}_{app}$ for any $\theta\in\Theta$ where $\mathcal{E}^{actor}_{app}$ is a positive constant.
\end{assumption}

Assumption \ref{assump_transfer_error} bounds the error that arises from the policy class parameterization. For neural networks, $\mathcal{E}^{actor}_{app}$ has shown to be small \citep{wang2019neural}, while for softmax policies, $\mathcal{E}^{actor}_{app} = 0$ \citep{agarwal2021theory}. This approximation assumption has also been used in \cite{bai2023regret}, and is important to generalizing the policy parameterization in the convergence analysis  we will later see in Section \ref{section:global_convergence} as $\mathcal{E}^{actor}_{app}$ will appear as its own independent term in the final bound.

For our later analysis, we also define 
\small
    \begin{align}
            F(\theta)=
            \mathbf{E}_{s\sim d^{\pi_{\theta}}}\mathbf{E}_{a\sim\pi_{\theta}(\cdot\vert s)}\left[\nabla_{\theta}\log\pi_{\theta}(a|s)(\nabla_{\theta}\log\pi_{\theta}(a|s))^T \right],\nonumber
    \end{align}
\normalsize
as the Fisher information matrix. We can now also express $h^*_\theta$ defined in \eqref{eq:NPG_direction} as,
    \begin{align*}
        h^*_{\theta} = F(\theta)^{\dagger} \mathbf{E}_{s\sim d^{\pi_{\theta}}}\mathbf{E}_{a\sim\pi_{\theta}(\cdot\vert s)}\left[\nabla_{\theta}\log\pi_{\theta}(a|s)A^{\pi_{\theta}}(s, a)\right],
    \end{align*}
where $\dagger$ denotes the Moore-Penrose pseudoinverse operation.

\begin{assumption}
    \label{assump_4}
    Setting $I_{F}$ as the identity matrix of same dimensionality as $F(\theta)$, let there exists some positive constant $\mu_F$ such that $F(\theta)-\mu_F I_{F}$ is positive semidefinite.
\end{assumption}
Assumption \ref{assump_4} is a common assumption for global convergence of policy gradient algorithms \cite{liu2020improved, bai2023regret}. This assumption will be useful later in our analysis by translating the general framework proposed in Lemma \ref{lem_framework} into terms of $\left[ \bignorm{ \nabla J(\theta_t) }^2 \right]$, the local convergence rate of MAC. We then can use the convergence rate bound established by \cite{suttle2023beyond}, which we state in Lemma \ref{lemma:convergence_rate}.

\begin{lemma} \label{lemma:31ss}
Let $j_{max} = \floor{\log T_{max}}$. Fix $\theta_t$ measurable w.r.t. $\mathcal{F}_{t-1}$. Assume $T_{max} \geq \tau_{mix}^{\theta_t}$, $\norm{\nabla J(\theta)} \leq G_H$, for all $\theta$, and $\norm{ h_t^N } \leq G_H$, for all $N \in \left[ T_{max} \right]$. Then
\small
\begin{align}
    \mathbf{E}_{t-1} \left[ h_t^{MLMC} \right] &= \mathbf{E}_{t-1} \left[ h_t^{j_{max}} \right], \label{eqn:31ss_mean} \\  
    \mathbf{E} \left[ \norm{ h_t^{MLMC} }^2 \right] &\leq \wo{ G_H^2 \tau_{mix}^{\theta_t} \log T_{max} } + 8 \log(T_{max}) T_{max} \left( \mathcal{E}(t) + 16B^2 (\mathcal{E}^{critic}_{app})^2 \right)  \label{eqn:31ss_2nd_moment} \\ 
    \mathbf{E} \norm{\nabla J(\theta_t)-h&_t^{j_{max}}}^2  \leq \wo{ G_H^2 \tau_{mix}^{\theta_t} \frac{\log T_{max}}{T_{max}} } +  \mathcal{E}_2(t) +  16B^2 (\mathcal{E}^{critic}_{app})^2. \label{eqn:est_err_variance}
\end{align}
\normalsize
\end{lemma}

Lemma \ref{lemma:31ss} provides a bound for the variance of the MLMC gradient estimator and will be integral to our analysis.


\begin{lemma} \label{thm:critic_analysis_main_body}
Let $\beta_t = \gamma_t = (1 + t)^{-\nu}, \alpha = \alpha_t' / \sqrt{\sum_{k=1}^t \norm{h^{MLMC}_t}^2}$, and $\alpha_t' = (1 + t)^{-\sigma}$, where $0 < \nu < \sigma < 1$. Then
\small
\begin{align}
    \frac{1}{T} &\sum_{t=1}^T  \mathcal{E}(t) \leq \bo{T^{\nu - 1}} + \bo{T^{-2(\sigma - \nu)}} 
    + \wo{  \tau_{mix} \log T_{\max}} \bo{T^{-\nu}} + \wo{  \tau_{mix} \frac{\log T_{\max}}{T_{\max}} }. 
    %
    %
    \label{ineq:critic_1}
\end{align}
\normalsize

By setting $\nu = 0.5$ and $\sigma = 0.75$ leads to the following:
\small
\begin{align}
    \frac{1}{T} \sum_{t=1}^T  \mathcal{E}(t) \leq 
    &\wo{ \tau_{mix} \log T_{\max}} \bo{T^{-\frac{1}{2}}} \nonumber\\&+ \wo{ \tau_{mix}\frac{\log T_{\max}}{T_{\max}} }. \label{ineq:critic_2}
\end{align}
\normalsize

\end{lemma}

Lemma \ref{thm:critic_analysis_main_body} established by \cite{suttle2023beyond} states the convergence rate of the critic with an MLMC estimator. This result directly affects the overall MAC convergence rate from \cite{suttle2023beyond} stated below,

\begin{lemma}(MAC Convergence Rate) \label{lemma:convergence_rate}
    Assume $J(\theta)$ is $L$-smooth, $\sup_{\theta} | J(\theta) | \leq M$, and $\norm{ \nabla J(\theta) }, \norm{ h_t^{MLMC} } \leq G_H$, for all $\theta, t$ and under assumptions of Lemma \ref{thm:critic_analysis_main_body}, we have
    \begin{align}\label{final_mac_bound}
    \frac{1}{T} \sum_{t=1}^T \mathbb{E}
     \left[ \bignorm{ \nabla J(\theta_t) }^2 \right]\leq \bo{ \mathcal{E}^{critic}_{app} }+\wo{ \frac{\tau_{mix} \log T_{\max}}{\sqrt{T}}}
 +\wo{ { \frac{\tau_{mix}\log T_{\max}}{T_{\max}}} } .
    \end{align}

\end{lemma}

Both Lemmas \ref{thm:critic_analysis_main_body} and \ref{lemma:convergence_rate} rely on the convergence of the error in average reward tracking, provided in Lemma D.1 of \cite{suttle2023beyond}. However, we have noticed that there is a $\wo{ \sqrt{ \frac{\tau_{mix}\log T_{\max}}{T_{\max}}} }$ term in Lemma D.1 that should actually absorb the $\wo{ { \frac{\tau_{mix}\log T_{\max}}{T_{\max}}} }$ term in \ref{ineq:critic_2} and \ref{final_mac_bound}. We provide a correct version of the proof of Lemma D.1 in the Appendix where we were able to remove the square root.

Finally, we will use the following result to manipulate the AdaGrad stepsizes in the final result of this section.
\begin{lemma}{Lemma 4.2, \citep{dorfman2022}.} \label{lemma:42}
For any non-negative real numbers $\{ a_i \}_{i \in [n]}$,
\begin{equation}
    \sum_{i=1}^n \frac{a_i}{ \sqrt{ \sum_{j=1}^i a_j } } \leq 2 \sqrt{ \sum_{i=1}^n a_i }.
\end{equation}
\end{lemma}


\subsection{Global Convergence Guarantee}
To develop our convergence analysis, we present a modified version of the general framework proposed in \cite{bai2023regret} to accommodate non-constant stepsizes such as Adagrad. 

\begin{lemma}
    \label{lem_framework} 
    Suppose a general gradient ascent algorithm updates the policy parameter in the following way.
    \begin{equation}
	\theta_{t+1}=\theta_t+\alpha_t h_t.
    \end{equation}
    When Assumptions \ref{assum:policy_conditions}, \ref{assump_transfer_error}, and \ref{assump_4} hold, we have the following inequality for any $T$.
    \begin{equation}
        \label{eq:general_bound}
	\begin{split}
            J^{*}-\frac{1}{T}\sum_{t=1}^{T}J(\theta_t)\leq \sqrt{\mathcal{E}^{actor}_{app}}+\frac{B}{T}\sum_{t}^{T}\Vert(h_t-h^*_t)\Vert
            +\frac{R}{2T}\sum_{t=1}^{T}\alpha_t\Vert h_t\Vert^2+\frac{1}{T}\sum_{t=1}^{T}\frac{1}{\alpha_t}\mathbf{E}_{s\sim d^{\pi^*}}\zeta_t,	
        \end{split}
    \end{equation}
    where $h^*_t:=h^*_{\theta_t}$ and $h^*_{\theta_t}$ is defined in \eqref{eq:NPG_direction}, $J^*=J(\theta^*)$, and $\pi^*=\pi_{\theta^*}$ where $\theta^*$ is the optimal parameter, and $\zeta_t = [KL(\pi^*(\cdot\vert s)\Vert\pi_{\theta_k}(\cdot\vert s))-KL(\pi^*(\cdot\vert s)\Vert\pi_{\theta_{k+1}}(\cdot\vert s))]$.
\end{lemma}

We provide a proof of the above lemma in Appendix \ref{lem_framework_proof}. The proof is similar to that of \cite{bai2023regret} with a notable difference that the non-constant stepsize does not allow us to simplify the telescoping summation in the last term without bounding $\alpha_t$ to some constant which we will do later in the analysis.


\begin{theorem}
\small
    \label{theorem_1_statement}
    Let $\{\theta_t\}_{t=1}^{T}$ be defined as in Lemma \ref{lem_framework}. If assumptions \ref{assumption:erogidcity}, \ref{assum:policy_conditions}, \ref{assump_transfer_error},  \ref{assump_4} hold,  $J(\cdot)$ is $L$-smooth, then the following inequality holds.
    \begin{align}
        \label{thm_final_convergence}
        \begin{split}
        J^{*}-\frac{1}{T}\sum_{t=1}^{T}\mathbf{E}\left[J(\theta_t)\right]\leq  \sqrt{\mathcal{E}^{actor}_{app}}   + \wo{ \frac{\sqrt{\tau_{mix} T_{\max}} \log T_{\max}}{T^{\frac{1}{2}}}}  
         +\wo{ \frac{\sqrt{\tau_{mix} \log T_{\max}}}{{T^{\frac{1}{4}}}}} +\mathcal{E}^{critic}_{app}
        +\wo{  \sqrt{\frac{\tau_{mix}\log T_{\max}}{T_{\max}}}}.
        \end{split}
    \end{align}
\normalsize
\end{theorem}
The proof of this above theorem can be found in Appendix \ref{proof_of_thm1}. Here, we provide a proof sketch to highlight the main mechanics.

\textit{Proof sketch.} We rewrite the bound of the expectation of \ref{eq:general_bound} of Lemma \ref{lem_framework} into terms of \ref{eqn:31ss_2nd_moment} of Lemma \ref{lemma:31ss} and \ref{final_mac_bound} of Lemma \ref{lemma:convergence_rate}. 

The expectation of the second term of the RHS of \ref{eq:general_bound} can be bounded as the following using Lemma \ref{lemma:31ss} and Assumption \ref{assump_4},
\small

\begin{equation}\label{eq:exp_second_term}
    \begin{split}
         \frac{1}{T}\sum_{t=1}^{T}\mathbf{E}\Vert h_t- h^*_t\Vert  \leq 
        \sqrt{\frac{1}{T}\sum_{t=1}^T \mathbf{E}\Vert h_t^{j_{max}} -\nabla J(\theta_t)\Vert^2} 
        +\sqrt{\frac{2}{T}\sum_{t=1}^{T}\left(2+\dfrac{1}{\mu_F^2}\right)\mathbf{E}\bigg[\Vert \nabla_\theta J(\theta_t)\Vert^2\bigg]}.
    \end{split}
\end{equation}

\normalsize

For the third term of the RHS of \ref{eq:general_bound}, we can utilize Lemma \ref{lemma:42}
\small
\begin{equation}\label{eq:exp_third_term}
    \begin{split}
        \frac{R}{2T}\sum_{t=1}^{T}\alpha_t\Vert h_t\Vert^2  \leq \frac{R}{T}\sqrt{\sum_{t=1}^{T}\Vert h_t\Vert^2}. 
    \end{split}
\end{equation}

\normalsize
We can also bound the fourth term using the fact that it is a telescoping sum and that $\alpha_T < \alpha_t$, $\alpha_T = \frac{\alpha'_T}{\sum_{t=1}^{T}\Vert h_t\Vert^2}$,

\small
\begin{equation}
    \label{eq:second_bound_with_adagrad}
    \begin{split}
        \frac{1}{T}\sum_{t=1}^{T}\frac{1}{\alpha_t}\mathbf{E}_{s\sim d^{\pi^*}}[\zeta_t]  \leq \frac{\mathbf{E}_{s\sim d^{\pi^*}}[KL(\pi^*(\cdot\vert s)\Vert\pi_{\theta_1}(\cdot\vert s))]}{\alpha'_T}\frac{1}{T}\sqrt{\sum_{t=1}^{T}\Vert h_t\Vert^2}.
    \end{split}
\end{equation}
\normalsize

\noindent We can now plug these bounds back into \ref{eq:general_bound} and ignore constants:. 

\small
\begin{equation}
    \label{eq:general_bound_expectation}
    \begin{split}
            J^{*}-\frac{1}{T}\sum_{t=1}^{T}\mathbf{E}\Vert J(\theta_t)\Vert
            \leq \sqrt{\mathcal{E}^{actor}_{app}} +\frac{1}{T}\sqrt{\sum_{t=1}^{T}\mathbf{E}\bigg[\Vert  h_t\Vert^2\bigg]}   + \sqrt{\frac{1}{T}\sum_{t=1}^T \mathbf{E}\Vert h_t^{j_{max}} -\nabla J(\theta_t)\Vert^2} +\sqrt{\frac{1}{T}\sum_{t=1}^{T}\mathbf{E}\bigg[\Vert \nabla_\theta J(\theta_t)\Vert^2\bigg]}. 
    \end{split}
\end{equation}
\normalsize

Bounding the the second and third term by the RHS with Lemmas \ref{lemma:31ss} and \ref{lemma:convergence_rate} respectively concludes the proof.

\textbf{Remark.} With Theorem \ref{theorem_1_statement}, we can recover the $\bo{T^{-\frac{1}{4}}}$ as in \cite{bai2023regret}. Furthermore, MAC has a tighter dependence on mixing time with $\wo{\sqrt{\tau_{mix}}}$ compared to the $\wo{\tau_{mix}^2}$ in \cite{bai2023regret} despite having no prior knowledge of mixing time due to the combination of MLMC gradient estimation and Adagrad stepsize. Similar to \cite{bai2023regret}, the independent $\mathcal{E}_{app}^{actor} \geq 0$ term accounts for the general policy parameterization. However, our bound has no dependence on hitting time like \cite{bai2023regret} as their dependence was a result of their advantage estimation algorithm described in Section \ref{section:ppgae}.

\textbf{Discussion on Practicality.} In this section, we want to highlight how practically feasible it is to implement MAC as compared to PPGAE. As mentioned previously, PPGAE defines $T$ as the sample budget of the entire training process, $K$ as the number of epochs, and $H = 16\tau_{hit}\tau_{mix}\sqrt{T}(\log(T))^2$ as the length of one epoch, thus $K = T/H$. Because $K$ represents the number of epochs, it must be a positive integer. For $K = T/H \leq 1$, that is equivalent to $\frac{\sqrt{T}}{(\log(T))^2} \leq 16 \tau_{hit}\tau_{mix}$. Even if $\tau_{hit} = 10$ and $\tau_{mix} = 1$, $H \approx 6.6 * 10^9$. Since $\tau_{hit}$ can become infinitely large and $\tau_{mix}$ grows with environment complexity, in practice $H$ would be much higher. As Figure \ref{fig:min_H} shows, if we set $\tau_{hit} = 10$, as mixing time increases, the minimum episode length, $H$, increases exponentially to satisfy $K > 1$. Even at $\tau_{mix} = 60$ the minimum $H$ is around $10^{14}$ samples.

\begin{figure}[h]
  \centering
  \includegraphics[width=0.6\columnwidth]{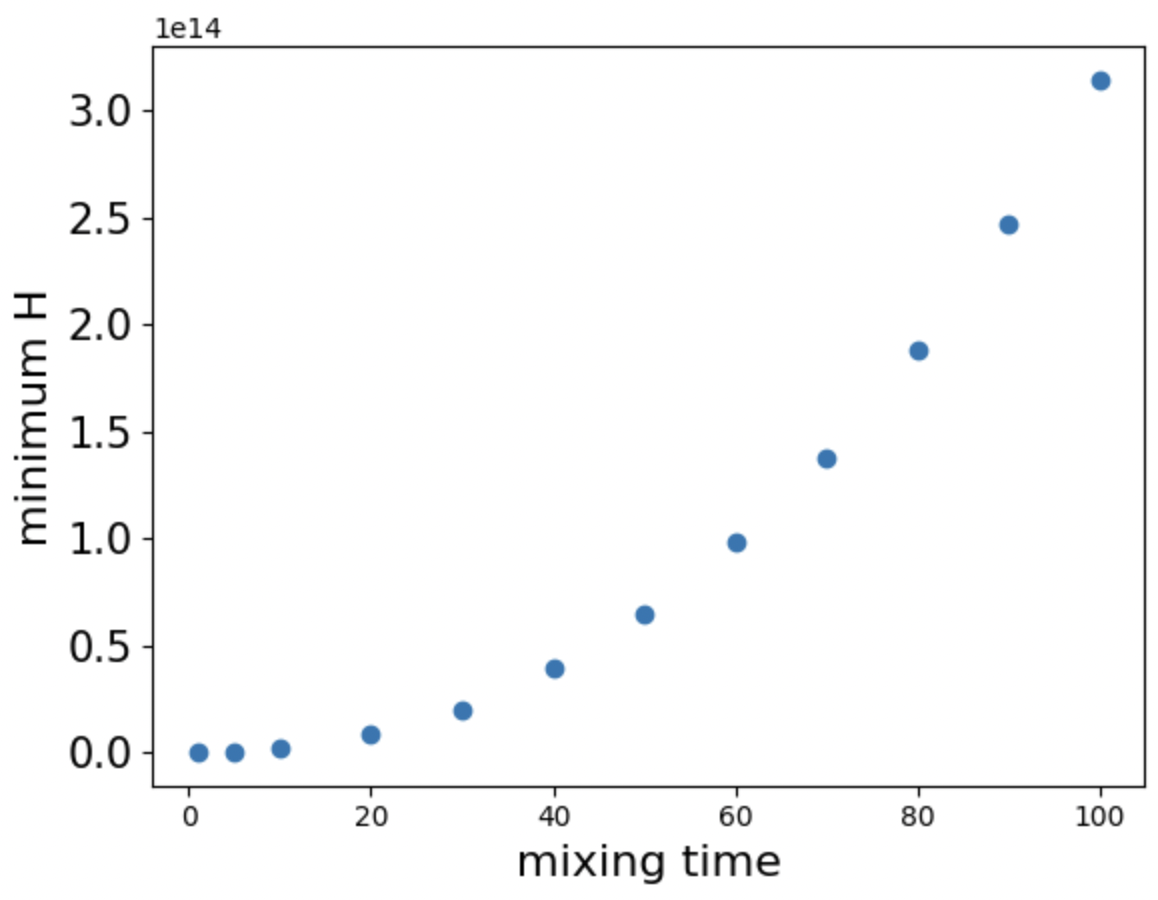}
  \caption{Minimum $H$ required for $K = 1$ given a mixing time $\tau_{mix}$. Both $H$ and $\tau_{\max}$ are in terms of number of samples. We set the hitting time to be 10 for this plot. }
  \label{fig:min_H}
\end{figure}

In contrast for MAC, the trajectory length is based on a geometric distribution with no dependence on mixing time, hitting time, or total sample budget. 


\subsection{Experimental Results}\label{section:experiments}

As a preliminary proof-of-concept experiment to show the advantage of MAC over PPGAE, Vanilla AC, and REINFORCE, we consider a $15$-by-$15$ sparse gridworld environment. The agent tries to from the top left to bottom right corner. The agent receives a reward of $+1$ if goal is reached and $+0$ else. The episode ends when the agent either reaches the goal or hits a limit of $200$ samples. We report a moving average success rate over $100$ trials with $95\%$ confidence intervals in Figure \ref{fig:experiments}. We can see that MAC has a higher success rate and can more consistently reach the goal than the baselines.

\begin{figure}[ht]
  \centering
  \includegraphics[width=0.6\columnwidth]{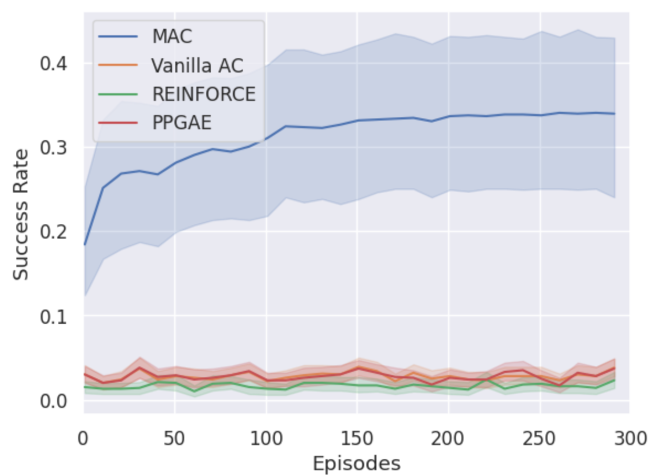}
  \caption{Success Rate in a sparse $15$-by-$15$ grid over 300 training episodes with $200$ samples per episode. For MAC, $T_{max} = 4$ and for PPGAE, $H = 200$ and $N = 1$. Vanilla AC and REINFORCE both have $H = 200$. and $100$ trials for each algorithm. PPGAE, Vanilla AC, and REINFORCE consistently converge to significantly less optimal solutions than MAC.}
  \label{fig:experiments}
\end{figure}






    

\section{Conclusion and Further Work}
In this work, we provide policy gradient global convergence analysis for the infinite horizon average reward MDP without restrictive and impractical assumptions on mixing time. Using MAC, we show that actor-critic models, utilizing a MLMC gradient estimator, achieves a tighter dependence on mixing time for global convergence. We hope this work encourages further investigation into algorithms that do not assume oracle knowledge of mixing time. Future work can also further test the advantages of MAC in slow mixing environments for robotics, finance, healthcare, and other applications.

\section*{Acknowledgements}
This research was supported by Army Cooperative Agreement W911NF2120076 and ARO Grant W911NF2310352  

\textbf{Disclaimer:} This paper was prepared for informational purposes in part by the Artificial Intelligence Research group of JPMorgan Chase \& Coand its affiliates (“JP Morgan”), and is not a product of the Research Department of JP Morgan. JP Morgan makes no representation and warranty whatsoever and disclaims all liability, for the completeness, accuracy or reliability of the information contained herein. This document is not intended as investment research or investment advice, or a recommendation, offer or solicitation for the purchase or sale of any security, financial instrument, financial product or service, or to be used in any way for evaluating the merits of participating in any transaction, and shall not constitute a solicitation under any jurisdiction or to any person, if such solicitation under such jurisdiction or to such person would be unlawful.





\bibliography{ref}

\begin{thebibliography}{32}
\providecommand{\natexlab}[1]{#1}
\providecommand{\url}[1]{\texttt{#1}}
\expandafter\ifx\csname urlstyle\endcsname\relax
  \providecommand{\doi}[1]{doi: #1}\else
  \providecommand{\doi}{doi: \begingroup \urlstyle{rm}\Url}\fi

\bibitem[Agarwal et~al.(2020)Agarwal, Kakade, Lee, and Mahajan]{agarwal2020optimality}
Agarwal, A., Kakade, S.~M., Lee, J.~D., and Mahajan, G.
\newblock Optimality and approximation with policy gradient methods in markov decision processes.
\newblock In \emph{Conference on Learning Theory}, pp.\  64--66, 2020.

\bibitem[Agarwal et~al.(2021)Agarwal, Kakade, Lee, and Mahajan]{agarwal2021theory}
Agarwal, A., Kakade, S.~M., Lee, J.~D., and Mahajan, G.
\newblock On the theory of policy gradient methods: Optimality, approximation, and distribution shift.
\newblock \emph{The Journal of Machine Learning Research}, 22\penalty0 (1):\penalty0 4431--4506, 2021.

\bibitem[Bai et~al.(2024)Bai, Mondal, and Aggarwal]{bai2023regret}
Bai, Q., Mondal, W.~U., and Aggarwal, V.
\newblock Regret analysis of policy gradient algorithm for infinite horizon average reward markov decision processes.
\newblock In \emph{AAAI Conference on Artificial Intelligence}, 2024.

\bibitem[Bedi et~al.(2022)Bedi, Chakraborty, Parayil, Sadler, Tokekar, and Koppel]{bedi2022hidden}
Bedi, A.~S., Chakraborty, S., Parayil, A., Sadler, B.~M., Tokekar, P., and Koppel, A.
\newblock On the hidden biases of policy mirror ascent in continuous action spaces.
\newblock In \emph{International Conference on Machine Learning}, pp.\  1716--1731, 2022.

\bibitem[Bhandari \& Russo(2024)Bhandari and Russo]{bhandari2019global}
Bhandari, J. and Russo, D.
\newblock Global optimality guarantees for policy gradient methods.
\newblock \emph{Operations Research}, 2024.

\bibitem[Dorfman \& Levy(2022)Dorfman and Levy]{dorfman2022}
Dorfman, R. and Levy, K.~Y.
\newblock Adapting to mixing time in stochastic optimization with {M}arkovian data.
\newblock In \emph{Proceedings of the 39th International Conference on Machine Learning}, Jul 2022.

\bibitem[Duchi et~al.(2011)Duchi, Hazan, and Singer]{duchi2011}
Duchi, J., Hazan, E., and Singer, Y.
\newblock Adaptive subgradient methods for online learning and stochastic optimization.
\newblock \emph{Journal of Machine Learning Research}, 12\penalty0 (61):\penalty0 2121--2159, 2011.
\newblock URL \url{http://jmlr.org/papers/v12/duchi11a.html}.

\bibitem[Duchi et~al.(2012)Duchi, Agarwal, Johansson, and Jordan]{duchi2012}
Duchi, J.~C., Agarwal, A., Johansson, M., and Jordan, M.~I.
\newblock Ergodic mirror descent.
\newblock \emph{SIAM Journal on Optimization}, 22\penalty0 (4):\penalty0 1549--1578, 2012.

\bibitem[Geng et~al.(2020)Geng, Lan, Aggarwal, Yang, and Xu]{geng2020multi}
Geng, N., Lan, T., Aggarwal, V., Yang, Y., and Xu, M.
\newblock A multi-agent reinforcement learning perspective on distributed traffic engineering.
\newblock In \emph{2020 IEEE 28th International Conference on Network Protocols (ICNP)}, pp.\  1--11. IEEE, 2020.

\bibitem[Gong \& Wang(2020)Gong and Wang]{gong2020duality}
Gong, H. and Wang, M.
\newblock A duality approach for regret minimization in average-award ergodic markov decision processes.
\newblock In \emph{Learning for Dynamics and Control}, 2020.

\bibitem[Hsu et~al.(2015)Hsu, Kontorovich, and Szepesv{\'a}ri]{hsu2015}
Hsu, D.~J., Kontorovich, A., and Szepesv{\'a}ri, C.
\newblock Mixing time estimation in reversible markov chains from a single sample path.
\newblock \emph{Advances in neural information processing systems}, 28, 2015.

\bibitem[Kumar et~al.(2019)Kumar, Koppel, and Ribeiro]{kumar2019sample}
Kumar, H., Koppel, A., and Ribeiro, A.
\newblock On the sample complexity of actor-critic method for reinforcement learning with function approximation.
\newblock \emph{arXiv preprint arXiv:1910.08412}, 2019.

\bibitem[Levy(2017)]{levy2017online}
Levy, K.
\newblock Online to offline conversions, universality and adaptive minibatch sizes.
\newblock \emph{Advances in Neural Information Processing Systems}, 30, 2017.

\bibitem[Ling et~al.(2023)Ling, Mondal, and Ukkusuri]{ling2023cooperating}
Ling, L., Mondal, W.~U., and Ukkusuri, S.~V.
\newblock Cooperating graph neural networks with deep reinforcement learning for vaccine prioritization.
\newblock \emph{arXiv preprint arXiv:2305.05163}, 2023.

\bibitem[Liu et~al.(2020)Liu, Zhang, Basar, and Yin]{liu2020improved}
Liu, Y., Zhang, K., Basar, T., and Yin, W.
\newblock An improved analysis of (variance-reduced) policy gradient and natural policy gradient methods.
\newblock \emph{Advances in Neural Information Processing Systems}, 33:\penalty0 7624--7636, 2020.

\bibitem[Mei et~al.(2020)Mei, Xiao, Szepesvari, and Schuurmans]{mei2020global}
Mei, J., Xiao, C., Szepesvari, C., and Schuurmans, D.
\newblock On the global convergence rates of softmax policy gradient methods.
\newblock In \emph{International Conference on Machine Learning}, pp.\  6820--6829, 2020.

\bibitem[Nagaraj et~al.(2020)Nagaraj, Wu, Bresler, Jain, and Netrapalli]{bresler2020}
Nagaraj, D., Wu, X., Bresler, G., Jain, P., and Netrapalli, P.
\newblock Least squares regression with markovian data: Fundamental limits and algorithms.
\newblock In \emph{Advances in Neural Information Processing Systems}, 2020.

\bibitem[Papini et~al.(2018)Papini, Binaghi, Canonaco, Pirotta, and Restelli]{papini2018stochastic}
Papini, M., Binaghi, D., Canonaco, G., Pirotta, M., and Restelli, M.
\newblock Stochastic variance-reduced policy gradient.
\newblock In \emph{International conference on machine learning}, pp.\  4026--4035, 2018.

\bibitem[Patel et~al.(2023)Patel, Weerakoon, Suttle, Koppel, Sadler, Bedi, and Manocha]{patel2023ada}
Patel, B., Weerakoon, K., Suttle, W.~A., Koppel, A., Sadler, B.~M., Bedi, A.~S., and Manocha, D.
\newblock Ada-nav: Adaptive trajectory-based sample efficient policy learning for robotic navigation.
\newblock \emph{arXiv preprint arXiv:2306.06192}, 2023.

\bibitem[Pesquerel \& Maillard(2022)Pesquerel and Maillard]{pesquerel2022imed}
Pesquerel, F. and Maillard, O.-A.
\newblock Imed-rl: Regret optimal learning of ergodic markov decision processes.
\newblock In \emph{NeurIPS 2022-Thirty-sixth Conference on Neural Information Processing Systems}, 2022.

\bibitem[Puterman(2014)]{puterman2014markov}
Puterman, M.~L.
\newblock \emph{Markov decision processes: discrete stochastic dynamic programming}.
\newblock John Wiley \& Sons, 2014.

\bibitem[Riemer et~al.(2021)Riemer, Raparthy, Cases, Subbaraj, Touzel, and Rish]{riemer2021continual}
Riemer, M., Raparthy, S.~C., Cases, I., Subbaraj, G., Touzel, M.~P., and Rish, I.
\newblock Continual learning in environments with polynomial mixing times.
\newblock \emph{arXiv preprint arXiv:2112.07066}, 2021.

\bibitem[Suttle et~al.(2023)Suttle, Bedi, Patel, Sadler, Koppel, and Manocha]{suttle2023beyond}
Suttle, W.~A., Bedi, A., Patel, B., Sadler, B.~M., Koppel, A., and Manocha, D.
\newblock Beyond exponentially fast mixing in average-reward reinforcement learning via multi-level monte carlo actor-critic.
\newblock In \emph{International Conference on Machine Learning}, pp.\  33240--33267, 2023.

\bibitem[Sutton \& Barto(2018)Sutton and Barto]{sutton2018reinforcement}
Sutton, R.~S. and Barto, A.~G.
\newblock \emph{Reinforcement learning: An introduction}.
\newblock MIT press, 2018.

\bibitem[Sutton et~al.(1999)Sutton, McAllester, Singh, and Mansour]{sutton1999policy}
Sutton, R.~S., McAllester, D., Singh, S., and Mansour, Y.
\newblock Policy gradient methods for reinforcement learning with function approximation.
\newblock \emph{Advances in neural information processing systems}, 12, 1999.

\bibitem[Wang et~al.(2019)Wang, Cai, Yang, and Wang]{wang2019neural}
Wang, L., Cai, Q., Yang, Z., and Wang, Z.
\newblock Neural policy gradient methods: Global optimality and rates of convergence.
\newblock In \emph{International Conference on Learning Representations}, 2019.

\bibitem[Wei et~al.(2021)Wei, Jahromi, Luo, and Jain]{wei2021learning}
Wei, C.-Y., Jahromi, M.~J., Luo, H., and Jain, R.
\newblock Learning infinite-horizon average-reward mdps with linear function approximation.
\newblock In \emph{International Conference on Artificial Intelligence and Statistics}, pp.\  3007--3015. PMLR, 2021.

\bibitem[Williams(1992)]{williams1992simple}
Williams, R.~J.
\newblock Simple statistical gradient-following algorithms for connectionist reinforcement learning.
\newblock \emph{Machine learning}, 8\penalty0 (3):\penalty0 229--256, 1992.

\bibitem[Wolfer(2020)]{wolfermixing2020}
Wolfer, G.
\newblock Mixing time estimation in ergodic markov chains from a single trajectory with contraction methods.
\newblock In \emph{Proceedings of the 31st International Conference on Algorithmic Learning Theory}, Feb 2020.

\bibitem[Xu et~al.(2020)Xu, Gao, and Gu]{xu2020improved}
Xu, P., Gao, F., and Gu, Q.
\newblock An improved convergence analysis of stochastic variance-reduced policy gradient.
\newblock In \emph{Uncertainty in Artificial Intelligence}, pp.\  541--551, 2020.

\bibitem[Zhang et~al.(2020)Zhang, Koppel, Zhu, and Basar]{zhang2020global}
Zhang, K., Koppel, A., Zhu, H., and Basar, T.
\newblock Global convergence of policy gradient methods to (almost) locally optimal policies.
\newblock \emph{SIAM Journal on Control and Optimization}, 58\penalty0 (6):\penalty0 3586--3612, 2020.

\bibitem[Zhang \& Ross(2021)Zhang and Ross]{zhang2021}
Zhang, Y. and Ross, K.~W.
\newblock On-policy deep reinforcement learning for the average-reward criterion.
\newblock In \emph{{ICML}}, volume 139 of \emph{Proceedings of Machine Learning Research}, pp.\  12535--12545. {PMLR}, 2021.

\end{thebibliography}
\bibliographystyle{icml2024}

\newpage
\appendix

\onecolumn

\tableofcontents
\newpage
\section*{Appendix}
\section{Proof of Lemma \ref{lem_framework}}
In this section, we provide a bound for the difference between the optimal reward and the cumulative reward observed up to trajectory $T$ that will be used as our general framework for the global convergence analysis. Our framework is a modification from \cite{bai2023regret} in that it can handle non-constant stepsizes. The framework provided in \cite{bai2023regret} is itself an average reward adaptation of the framework provided by \cite{liu2020improved} for the discounted reward setting. We first provide a supporting result in the form the average reward performance difference lemma:
\label{lem_framework_proof}
\begin{lemma}
    \label{lem_performance_diff}
    The difference in the performance for  any policies $\pi_\theta$ and $\pi_{\theta'}$is bounded as follows
    \begin{equation}
        J(\theta)-J(\theta')= \mathbb{E}_{s\sim d^{\pi_\theta}}\mathbb{E}_{a\sim\pi_\theta(\cdot\vert s)}\big[A^{\pi_{\theta'}}(s,a)\big]
    \end{equation}
\end{lemma}

We can now provide the general framework lemma.

\begin{lemma} 

    Suppose a general gradient ascent algorithm updates the policy parameter in the following way.
    \begin{equation}
	\theta_{t+1}=\theta_t+ \alpha_t h_t
    \end{equation}
    When Assumptions \ref{assum:policy_conditions}, \ref{assump_transfer_error}, and \ref{lem_performance_diff} hold, we have the following inequality for any $T$.
    \begin{equation}
	\begin{split}
            &J^{*}-\frac{1}{T}\sum_{t=1}^{T}J(\theta_t)\leq \sqrt{\mathcal{E}^{actor}_{app}}+\frac{B}{T}\sum_{t}^{T}\Vert(h_t-h^*_t)\Vert+\frac{K}{2T}\sum_{t=1}^{T}\alpha_t\Vert h_t\Vert^2+\frac{1}{T}\sum_{t=1}^{T}\frac{1}{\alpha_t}\mathbb{E}_{s\sim d^{\pi^*}}\zeta_t	
        \end{split}
    \end{equation}
    where $h^*_t:=h^*_{\theta_t}$ and $h^*_{\theta_t}$ is defined in \eqref{eq:NPG_direction}, $J^*=J(\theta^*)$, and $\pi^*=\pi_{\theta^*}$ where $\theta^*$ is the optimal parameter, and $\zeta_t = [KL(\pi^*(\cdot\vert s)\Vert\pi_{\theta_t}(\cdot\vert s))-KL(\pi^*(\cdot\vert s)\Vert\pi_{\theta_{t+1}}(\cdot\vert s))]$.
\end{lemma}

\begin{proof}
    We start the proof by lower bounding the difference between the KL divergence between $\pi*$ and $\pi_{\theta}$ and the KL divergence between $\pi*$ and $\pi_{\theta + 1}$.
	\begin{align}
            &\mathbb{E}_{s\sim d^{\pi^*}}[KL(\pi^*(\cdot\vert s)\Vert\pi_{\theta_t}(\cdot\vert s))-KL(\pi^*(\cdot\vert s)\Vert\pi_{\theta_{t+1}}(\cdot\vert s))]\\
            &=\mathbb{E}_{s\sim d^{\pi^*}}\mathbb{E}_{a\sim\pi^*(\cdot\vert s)}\bigg[\log\frac{\pi_{\theta_{t+1}(a\vert s)}}{\pi_{\theta_t}(a\vert s)}\bigg]\\
            &\overset{(a)}\geq\mathbb{E}_{s\sim d^{\pi^*}}\mathbb{E}_{a\sim\pi^*(\cdot\vert s)}[\nabla_\theta\log\pi_{\theta_t}(a\vert s)\cdot(\theta_{t+1}-\theta_t)]-\frac{K}{2}\Vert\theta_{t+1}-\theta_t\Vert^2\\
            &= \mathbb{E}_{s\sim d^{\pi^*}}\mathbb{E}_{a\sim\pi^*(\cdot\vert s)}[\nabla_{\theta}\log\pi_{\theta_t}(a\vert s)\cdot \alpha_t h_t]-\frac{K \alpha_t^2}{2}\Vert h_t\Vert^2\\
            &= \mathbb{E}_{s\sim d^{\pi^*}}\mathbb{E}_{a\sim\pi^*(\cdot\vert s)}[\nabla_\theta\log\pi_{\theta_t}(a\vert s)\cdot \alpha_t h^*_t]+ \mathbb{E}_{s\sim d^{\pi^*}}\mathbb{E}_{a\sim\pi^*(\cdot\vert s)}[\nabla_\theta\log\pi_{\theta_t}(a\vert s)\cdot \alpha_t( h_t- h^*_t)]-\frac{K \alpha_t^2}{2}\Vert h_t\Vert^2\\
            &= \alpha_t[J^{*}-J(\theta_t)]+ \mathbb{E}_{s\sim d^{\pi^*}}\mathbb{E}_{a\sim\pi^*(\cdot\vert s)}[\nabla_\theta\log\pi_{\theta_t}(a\vert s)\cdot \alpha_t h^*_t]- \alpha_t[J^{*}-J(\theta_t)]\\
            &\quad + \mathbb{E}_{s\sim d^{\pi^*}}\mathbb{E}_{a\sim\pi^*(\cdot\vert s)}[\nabla_\theta\log\pi_{\theta_t}(a\vert s)\cdot \alpha_t( h_t- h^*_t)]-\frac{K \alpha_t^2}{2}\Vert h_t\Vert^2\\		&\overset{(b)}= \alpha_t[J^{*}-J(\theta_t)]+ \alpha_t\mathbb{E}_{s\sim d^{\pi^*}}\mathbb{E}_{a\sim\pi^*(\cdot\vert s)}\bigg[\nabla_\theta\log\pi_{\theta_t}(a\vert s)\cdot h^*_t-A^{\pi_{\theta_t}}(s,a)\bigg]\\
            &\quad 
            + \alpha_t\mathbb{E}_{s\sim d^{\pi^*}}\mathbb{E}_{a\sim\pi^*(\cdot\vert s)}[\nabla_\theta\log\pi_{\theta_t}(a\vert s)\cdot( h_t- h^*_t)]-\frac{K \alpha_t^2}{2}\Vert h_t\Vert^2
            \end{align}
            {Taking the conditional expectation in the above expression, and using the equality in \eqref{eqn:31ss_mean}, we can write
            \begin{align}
                &\mathbb{E}_{t}\mathbb{E}_{s\sim d^{\pi^*}}[KL(\pi^*(\cdot\vert s)\Vert\pi_{\theta_t}(\cdot\vert s))-KL(\pi^*(\cdot\vert s)\Vert\pi_{\theta_{t+1}}(\cdot\vert s))]\\
            &\geq  \alpha_t[J^{*}-\mathbb{E}_{t}[J(\theta_t)]]+ \alpha_t\mathbb{E}_{t}\mathbb{E}_{s\sim d^{\pi^*}}\mathbb{E}_{a\sim\pi^*(\cdot\vert s)}\bigg[\nabla_\theta\log\pi_{\theta_t}(a\vert s)\cdot h^*_t-A^{\pi_{\theta_t}}(s,a)\bigg]\\
            &\quad 
            + \alpha_t\mathbb{E}_{s\sim d^{\pi^*}}\mathbb{E}_{a\sim\pi^*(\cdot\vert s)}[\nabla_\theta\log\pi_{\theta_t}(a\vert s)\cdot( h_t^{j_{\max}}- h^*_t)]-\frac{K \alpha_t^2}{2}\mathbb{E}_{t}\Vert h_t\Vert^2
            \end{align}}
            
            \begin{align}
            &\overset{(c)}\geq \alpha_t[J^{*}-J(\theta_t)]- \alpha_t\sqrt{\mathbb{E}_{s\sim d^{\pi^*}}\mathbb{E}_{a\sim\pi^*(\cdot\vert s)}\bigg[\bigg(\nabla_\theta\log\pi_{\theta_t}(a\vert s)\cdot h^*_t-A^{\pi_{\theta_t}}(s,a)\bigg)^2\bigg]}\\
            &- \alpha_t\mathbb{E}_{s\sim d^{\pi^*}}\mathbb{E}_{a\sim\pi^*(\cdot\vert s)}\Vert\nabla_\theta\log\pi_{\theta_t}(a\vert s)\Vert_2\Vert( {h_t^{j_{\max}}}- h^*_t)\Vert-\frac{K \alpha_t^2}{2}\Vert h_t\Vert^2\\
            &\overset{(d)}\geq \alpha_t[J^{*}-J(\theta_t)]- \alpha_t\sqrt{\mathcal{E}^{actor}_{app}}- \alpha_t B\Vert( {h_t^{j_{\max}}}- h^*_t)\Vert-\frac{K \alpha_t^2}{2}\Vert h_t\Vert^2\\
	\end{align}	
    where we use Assumption \ref{assum:policy_conditions} for step (a) and Lemma \ref{lem_performance_diff} for step (b). Step (c) uses the convexity of the function $f(x)=x^2$, and (d) comes from Assumption \ref{assump_transfer_error}. We can get by rearranging terms,
    \begin{equation}
	\begin{split}
            J^{*}-J(\theta_t)\leq & \sqrt{\mathcal{E}^{actor}_{app}}+ B\Vert( h_t^{j_{\max}} - h^*_t)\Vert+\frac{K \alpha_t}{2}\Vert h_t\Vert^2\\
            &+\frac{1}{ \alpha_t}\mathbb{E}_{s\sim d^{\pi^*}}[KL(\pi^*(\cdot\vert s)\Vert\pi_{\theta_t}(\cdot\vert s))-KL(\pi^*(\cdot\vert s)\Vert\pi_{\theta_{t+1}}(\cdot\vert s))]
	\end{split}
    \end{equation}
    Because KL divergence is either $0$ or positive, we can conclude the proof by taking the average over $T$ trajectories. 
\end{proof}
\section{Proof of Theorem \ref{theorem_1_statement}}\label{proof_of_thm1}
To use \ref{eq:general_bound} for our convergence analysis, we will take the expectation of the second term. With $h_t^{MLMC} = h_t$, note that,
\begin{equation}
    \begin{split}
        \bigg(\frac{1}{T}\sum_{t=1}^{T}\mathbb{E}\Vert {h_t^{j_{\max}}}- h^*_t\Vert\bigg)^2\leq &\frac{1}{T}\sum_{t=1}^{T}\mathbb{E}\bigg[\Vert {h_t^{j_{\max}}}-h^*_t\Vert^2\bigg]\nonumber\\
        =&\frac{1}{T}\sum_{t=1}^{T}\mathbb{E}\bigg[\Vert {h_t^{j_{\max}}}-F(\theta_t)^\dagger\nabla_\theta J(\theta_t)\Vert^2\bigg]\nonumber\\
        \leq& \frac{2}{T}\sum_{t=1}^{T}\mathbb{E}\bigg[\Vert {h_t^{j_{\max}}}-\nabla_{\theta}J(\theta_t)\Vert^2\bigg]+\frac{2}{T}\sum_{t=1}^{T}\mathbb{E}\bigg[\Vert \nabla_{\theta} J(\theta_t)- F(\theta_t)^\dagger\nabla_\theta J(\theta_t)\Vert^2\bigg] \nonumber\\
        \overset{(a)}{\leq} &\frac{2}{T}\sum_{t=1}^{T}\mathbb{E}\bigg[\Vert  {h_t^{j_{\max}}}-\nabla_{\theta}J(\theta_t)\Vert^2\bigg]+\frac{2}{T}\sum_{t=1}^{T}\left(1+\dfrac{1}{\mu_F^2}\right)\mathbb{E}\bigg[\Vert \nabla_\theta J(\theta_t)\Vert^2\bigg],
    \end{split}
\end{equation}
where $(a)$ uses Assumption \ref{assump_4}. Taking the square root of both sides and from $\sqrt{a + b} \leq \sqrt{a} + \sqrt{b}$ we arrive at:
\begin{equation}
    \label{eq_second_term_bound_sqrt}
    \begin{split}
        \bigg(\frac{1}{T}\sum_{t=1}^{T}\mathbb{E}\Vert {h_t^{j_{\max}}}- h^*_t\Vert\bigg)\leq  &\sqrt{\frac{2}{T}\sum_{t=1}^{T}\mathbb{E}\bigg[\Vert  {h_t^{j_{\max}}}-\nabla_{\theta}J(\theta_t)\Vert^2\bigg]}+\sqrt{\frac{2}{T}\sum_{t=1}^{T}\left(1+\dfrac{1}{\mu_F^2}\right)\mathbb{E}\bigg[\Vert \nabla_\theta J(\theta_t)\Vert^2\bigg]},
    \end{split}
\end{equation}

We can also bound the third term of the RHS of \ref{eq:general_bound} with Lemma \ref{lemma:42}
\small
\begin{equation}
    \label{eq:first_bound_with_adagrad}
    \begin{split}
        \frac{R}{2T}\sum_{t=1}^{T}\alpha_t\Vert h_t\Vert^2  \leq \frac{R}{2T}\sum_{t=1}^{T}\frac{\Vert h_t\Vert^2}{\sqrt{\sum_{o=1}^{t}\Vert h_o\Vert^2}} \leq \frac{R}{T}\sqrt{\sum_{t=1}^{T}\Vert h_t\Vert^2}
    \end{split}
\end{equation}
\normalsize
We can also bound the fourth term using the fact that it is a telescoping sum and that $\alpha_T < \alpha_t$,
\small
\begin{equation}\label{eq:app_second_bound_with_adagrad}
    \begin{split}
        \frac{1}{T}\sum_{t=1}^{T}\frac{1}{\alpha_t}\mathbb{E}_{s\sim d^{\pi^*}}[\zeta_t] \leq& \frac{1}{T}\sum_{t=1}^{T}\frac{\mathbb{E}_{s\sim d^{\pi^*}}[\zeta_t]}{\alpha_T}\ \\=& \frac{1}{T}\frac{\mathbb{E}_{s\sim d^{\pi^*}}[KL(\pi^*(\cdot\vert s)\Vert\pi_{\theta_1}(\cdot\vert s))]}{\alpha_T} \\\leq& \frac{\mathbb{E}_{s\sim d^{\pi^*}}[KL(\pi^*(\cdot\vert s)\Vert\pi_{\theta_1}(\cdot\vert s))]\sqrt{\sum_{t=1}^{T}\Vert h_t\Vert^2}}{T\alpha'_T} .
    \end{split}
\end{equation}
\normalsize


Taking the expectation of both sides of  \ref{eq:general_bound} and plugging in \ref{eq_second_term_bound_sqrt}, \ref{eq:first_bound_with_adagrad}, and \ref{eq:app_second_bound_with_adagrad}, ignoring constants:
\small
\begin{equation}
    \begin{split}
            J^{*}-\frac{1}{T}\sum_{t=1}^{T}\mathbb{E}\Vert J(\theta_t)\Vert
            \leq &\sqrt{\mathcal{E}^{actor}_{app}}   + \sqrt{\frac{1}{T}\sum_{t=1}^T \mathbf{E}\Vert h_t^{j_{max}} -\nabla J(\theta_t)\Vert^2} \\&+\sqrt{\frac{1}{T}\sum_{t=1}^{T}\mathbf{E}\bigg[\Vert \nabla_\theta J(\theta_t)\Vert^2\bigg]} +\frac{1}{T}\sqrt{\sum_{t=1}^{T}\mathbf{E}\bigg[\Vert  h_t\Vert^2\bigg]} 
            .
    \end{split}
\end{equation}
\normalsize
From Lemma \ref{lemma:31ss} we can bound the summation of the expected variance of the MLMC gradient. Ignoring the $G_H$ constant,
\small
\begin{equation}
    \begin{split}
         \sum_{t=1}^{T}\mathbb{E}\bigg[\Vert  h_t\Vert^2\bigg]  &\leq \sum_{t=1}^{T}\wo{ \tau_{mix}^{\theta_t} \log T_{max} } + \sum_{t=1}^{T} \log(T_{max}) T_{max}\mathcal{E}_2(t)+\sum_{t=1}^{T} \log(T_{max}) T_{max} (\mathcal{E}^{critic}_{app})^2 
    \end{split}
\end{equation}
\normalsize
We can bound the third term of the RHS by utilizing the maximum mixing time, $\tau_{mix}$
\begin{equation}
    \begin{split}
        \sum_{t=1}^{T}\wo{ \tau_{mix}^{\theta_t} \log T_{max} } \leq \wo{ T\tau_{mix} \log T_{max} }
    \end{split}
\end{equation}

For the second term by using \ref{ineq:critic_2}, 
\small
\begin{equation}
    \begin{split}
        \sum_{t=1}^{T} \log(T_{max}) T_{max}\mathcal{E}(t)  
     \leq & T(\log T_{\max})T_{\max}\wo{ \tau_{mix} (\log T_{\max})} \bo{T^{-\frac{1}{2}}} \\ & + T(\log T_{\max})T_{\max}\wo{ \tau_{mix}\frac{(\log T_{\max})}{T_{\max}} }\\ 
         =&\wo{T \tau_{mix} (\log T_{\max})^2T_{\max}}.
    \end{split}
\end{equation}
\normalsize
The third term can be simply bounded as follows:
\small
\begin{equation}
    \begin{split}
       \sum_{t=1}^{T} \log(T_{max}) T_{max} (\mathcal{E}^{critic}_{app})^2 \leq  \bo{T\log(T_{max}) T_{max} (\mathcal{E}^{critic}_{app})^2}.
    \end{split}
\end{equation}
\normalsize
We can now bound the summation of the expected variance of the MLMC gradient:
\small
\begin{equation}\label{variance_bound}
    \begin{split}
         \frac{1}{T}\sum_{t=1}^{T}\mathbb{E}\bigg[\Vert  h_t\Vert^2\bigg] 
         &\leq
         \wo{T \tau_{mix} \log T_{max} } + \wo{ T\tau_{mix} (\log T_{\max})^2T_{\max}} +\bo{T\log(T_{max}) T_{max}(\mathcal{E}^{critic}_{app})^2}. 
    \end{split}
\end{equation}

\normalsize

Taking the square root of both sides, from $\sqrt{a + b} \leq \sqrt{a} + \sqrt{b}$, and dividing by $T$: 
\small

\begin{equation}\label{variance_bound_sqrt_with_rt_T}
    \begin{split}
         \frac{1}{T}\sqrt{\sum_{t=1}^{T}\mathbb{E}\bigg[\Vert  h_t\Vert^2\bigg]}  
         &\leq  \wo{ \frac{\sqrt{\tau_{mix} \log T_{max} }}{T^{\frac{1}{2}}}} + \wo{ \frac{\sqrt{\tau_{mix} T_{\max}} \log T_{\max}}{T^{\frac{1}{2}}}} 
          + \bo{\frac{\sqrt{\log(T_{max}) T_{max}} \mathcal{E}^{critic}_{app}}{T^{\frac{1}{2}}}} 
    \end{split}
\end{equation}
\normalsize
From Lemma \ref{lemma:convergence_rate} we can bound the square root of the summation of the expectation of the gradient norm squared:
\small

\begin{equation}\label{convergence_rate_bound}
    \begin{split}
    &\sqrt{\frac{1}{T} \sum_{t=1}^T \mathbb{E}
     \left[ \bignorm{ \nabla J(\theta_t) }^2 \right]} \leq \bo{ \sqrt{\mathcal{E}^{critic}_{app}} }
        + \wo{ \frac{\sqrt{\tau_{mix} \log T_{\max}}}{{T^{\frac{1}{4}}}}} 
     + \wo{ { \frac{\sqrt{\tau_{mix}\log T_{\max}}}{\sqrt{T_{\max}}}}} 
    \end{split}
\end{equation}
\normalsize
We can also bound the error between the first term with Equation \ref{eqn:est_err_variance} and Lemma \ref{lemma:31ss}:

\small

\begin{equation}\label{app_estimation_error_bound}
    \begin{split}
    &\sqrt{\frac{1}{T}\sum_{t=1}^T \mathbf{E}\Vert h_t^{j_{max}} -\nabla J(\theta_t)\Vert^2} \leq \bo{ \mathcal{E}^{critic}_{app} }
        + \wo{ \frac{\sqrt{\tau_{mix} \log T_{\max}}}{{T^{\frac{1}{4}}}}} 
     + \wo{ { \frac{\sqrt{\tau_{mix}\log T_{\max}}}{\sqrt{T_{\max}}}}} 
    \end{split}
\end{equation}
\normalsize

We can see that all terms of \eqref{app_estimation_error_bound} absorb the terms of \eqref{convergence_rate_bound}. Combining \eqref{variance_bound_sqrt_with_rt_T} and \eqref{app_estimation_error_bound} we can get the final global convergence.

\section{Corrected Analysis of Multi-level Monte Carlo}
In this section, we wish to provide a corrected analysis for Lemma \ref{lemma:convergence_rate}. The issue lies in the convergence rate of the average reward tracker. We first give an overview of the problem and how it affects Lemma \ref{lemma:convergence_rate}. We then provide a corrected version of the average reward tracking analysis.

\subsection{Overview of Correction}

We repeat Lemma \ref{lemma:convergence_rate},

\begin{lemma}
    Assume $J(\theta)$ is $L$-smooth, $\sup_{\theta} | J(\theta) | \leq M$, and $\norm{ \nabla J(\theta) }, \norm{ h_t^{MLMC} } \leq G_H$, for all $\theta, t$ and under assumptions of Lemma \ref{thm:critic_analysis_main_body}, we have
    \begin{align}\label{final_bound}
    \frac{1}{T} \sum_{t=1}^T \mathbb{E}
     \left[ \bignorm{ \nabla J(\theta_t) }^2 \right]\leq & \bo{ \mathcal{E}^{critic}_{app} }+\wo{ \frac{\tau_{mix} \log T_{\max}}{\sqrt{T}}}  +\wo{ { \frac{\tau_{mix}\log T_{\max}}{T_{\max}}} } .
    \end{align}
\end{lemma}

The above lemma is correct. However, the analysis for it does not match this final statement. Specifically, given the current analysis provided in \cite{suttle2023beyond}, the $\wo{ { \frac{\tau_{mix}\log T_{\max}}{T_{\max}}} }$ term should actually be $\wo{ \sqrt{ \frac{\tau_{mix}\log T_{\max}}{T_{\max}}}}$. The term stems from Lemma \ref{thm:critic_analysis_main_body}, the convergence of the critic estimation $\mathcal{E}(t)$, which we repeat here,
\begin{lemma} 
Let $\beta_t = \gamma_t = (1 + t)^{-\nu}, \alpha = \alpha_t' / \sqrt{\sum_{k=1}^t \norm{h^{MLMC}_t}^2}$, and $\alpha_t' = (1 + t)^{-\sigma}$, where $0 < \nu < \sigma < 1$. Then
\small
\begin{align}
    \frac{1}{T} \sum_{t=1}^T  \mathcal{E}(t) \leq \bo{T^{\nu - 1}} + \bo{T^{-2(\sigma - \nu)}} 
    + \wo{  \tau_{mix} \log T_{\max}} \bo{T^{-\nu}} 
    %
    %
    + \wo{  \tau_{mix} \frac{\log T_{\max}}{T_{\max}} }. 
\end{align}
\end{lemma}
Once again the $\wo{ { \frac{\tau_{mix}\log T_{\max}}{T_{\max}}} }$ term should actually be $\wo{ \sqrt{ \frac{\tau_{mix}\log T_{\max}}{T_{\max}}}}$ based on the current analysis. This term from Lemma \ref{thm:critic_analysis_main_body} is dependent on the average reward tracking error. We repeat its convergence theorem from \cite{suttle2023beyond} below,
\begin{theorem} \label{thm:reward_analysis_app_orig}
Let $\beta_t = \gamma_t = (1 + t)^{-\nu}, \alpha = \alpha_t' / \sqrt{\sum_{k=1}^t \norm{h_t}^2}$, and $\alpha_t' = (1 + t)^{-\sigma}$, where $0 < \nu < \sigma < 1$. 
Then
\begin{align}
    \frac{1}{T} \sum_{t=1}^T \E \left[ (\eta_t - \eta_t^*)^2 \right] \leq \bo{T^{\nu - 1}} + \bo{T^{-2(\sigma - \nu)}} 
    + \wo{  \tau_{mix} \log T_{max}} \bo{T^{-\nu}} 
    + \wo{ \sqrt{  \tau_{mix} \frac{\log T_{max}}{T_{max}} } }.
\end{align}
\end{theorem}
The proof of Theorem \ref{thm:reward_analysis_app_orig} matches the statement above. In the next subsection we provide a correct version of the statement along with a proof that will align with Lemmas \ref{thm:critic_analysis_main_body} and \ref{lemma:convergence_rate}.
\subsection{Corrected Average Reward Tracking Error Analysis}

Before we provide the correct version of Theorem \ref{thm:reward_analysis_app_orig}, we provide the following lemma from \cite{dorfman2022} and utilized by \cite{suttle2023beyond} for Theorem \ref{thm:reward_analysis_app_orig} as we will still use it for the correct version of the theorem. In \cite{suttle2023beyond}, the following lemma is written for MLMC gradient estimator in general. Our restatement is tailored to the MLMC gradient estimation of the reward tracking error. 

\begin{lemma}{Lemma A.6, \cite{dorfman2022}.} \label{lemma:a6}
Given a policy $\pi_{\theta}$, assume we the trajectory sampled from it is $z_t = \{ z_t^i = (s_t^i, a_t^i, r_t^i, s_t^{i+1}) \}_{i \in [N]}$ starting from $s_t^0 \sim \mu_0(\cdot)$, where $\mu_0$ is the initial state distribution. Let $\nabla F(\eta)$ be an average reward tracking gradient that we wish to estimate over $z_t$, where $\mathbb{E}_{z \sim \mu_{\theta_t}, \pi_{\theta_t}} \left[ f(\eta, z) \right] = \nabla F(x)$, and $\eta \in \mathcal{K} \subset \mathbb{R}^k$ is the parameter of the estimator. Finally, assume that $\norm{ f(\eta, z) }, \norm{ \nabla F(\eta) } \leq 1$, for all $\eta \in \mathcal{K}, z \in \mathcal{S} \times \mathcal{A} \times \mathbb{R} \times \mathcal{S}$. Define $f_t^N = \frac{1}{N} \sum_{i=1}^N f(\eta_t, z_t^i)$. Fix $T_{max} \in \mathbb{N}$ and let $K = \tau_{mix} \lceil 2 \log T_{max} \rceil$. Then, for every $N \in \left[ T_{max} \right]$ and every $\eta_t \in \mathcal{K}$ measurable w.r.t. $\mathcal{F}_{t-1} = \sigma(\theta_k, \eta_k, \omega_k, z_k; k \leq t-1)$, where $\theta$ and $\omega$ are the parameters of the actor and critic, respectively,
\begin{align}
    \mathbb{E} \left[ \norm{ f_t^N - \nabla F(\eta_t) } \right] &\leq O \left(  \sqrt{ \log KN } \sqrt{ \frac{K}{N} } \right), \\
    \mathbb{E} \left[ \norm{ f_t^N - \nabla F(\eta_t) }^2 \right] &\leq O \left(  \log (KN) \frac{K}{N} \right).
\end{align}
\end{lemma}

Below is the corrected theorem for the reward tracking error analysis and the accompanying proof.
\begin{theorem} \label{thm:reward_analysis_app}
Assume $\gamma_t = (1 + t)^{-\nu}, \alpha = \alpha_t' / \sqrt{\sum_{k=1}^t \norm{h_t}^2}$, and $\alpha_t' = (1 + t)^{-\sigma}$, where $0 < \nu < \sigma < 1$. 
Then
\begin{align}
    \frac{1}{T} \sum_{t=1}^T \E \left[ (\eta_t - \eta_t^*)^2 \right] &\leq \bo{T^{\nu - 1}} + \bo{T^{-2(\sigma - \nu)}} \\
    &+ \wo{  \tau_{mix} \log T_{max}} \bo{T^{-\nu}} \\
    &+ \wo{   \tau_{mix} \frac{\log T_{max}}{T_{max} } }.
\end{align}
\end{theorem}
\begin{proof}
Because the proof closely resembles the original version from \cite{suttle2023beyond} with a few changes, we will only show intermediate steps for portions the changes affect. Similar to \cite{suttle2023beyond}, we recall that the average reward tracking update is given by
\begin{align}
	\eta_{t+1}=\eta_t - \gamma_t f_t, 
\end{align}
where $f_t := f_t^{\text{MLMC}}$. We can rewrite the tracking error term $(	\eta_{t+1}-\eta_{t+1}^*)^2$ as
\begin{align}
	(	\eta_{t+1}-\eta_{t+1}^*)^2 	&\leq 
	 (1- 2\gamma_t) (\eta_{t}-\eta_{t}^*)^2+ 2\gamma_t (\eta_{t}-\eta_{t}^*)(F' (\eta_{t})-f_t) + 2(	\eta_{t}-\eta_{t}^*) (\eta_{t}^*-\eta_{t+1}^*) 
	\nonumber
	\\
	&+2(\eta_{t}^*-\eta_{t+1}^*)^2+2(\gamma_tf_t)^2.\label{ref_4}
\end{align}
As in \cite{suttle2023beyond}, we take expectations and transform the expression into five separate summations,
\begin{align}
	\sum_{t=1}^{T}\E[(\eta_{t}-\eta_{t}^*)^2] 	\leq  & \underbrace{\sum_{t=1}^{T}\frac{1}{2\gamma_t} 	\E[(\eta_{t}-\eta_{t}^*)^2-(\eta_{t}-\eta_{t}^*)^2]}_{I_1}+  \underbrace{\sum_{t=1}^{T}	\E[(\eta_{t}-\eta_{t}^*)(F' (\eta_{t})-f_t)] }_{I_{2}}	\nonumber
	\\
	&+ \underbrace{\sum_{t=1}^{T}\frac{1}{\gamma_t}\E[(\eta_{t}-\eta_{t}^*) (\eta_{t}^*-\eta_{t+1}^*) ]}_{I_{3}}
+ \underbrace{\sum_{t=1}^{T}\frac{1}{\gamma_t}	\E[(\eta_{t}^*-\eta_{t+1}^*)^2]}_{I_4}+ \underbrace{\sum_{t=1}^{T}\gamma_t\E[(f_t)^2]}_{I_5}.\label{ref_5}
\end{align}
%
%
\cite{suttle2023beyond} provides bounds for $I_1, I_2, I_3, I_4$ and $I_5$. In this proof, only $I_2$ needs to be modified. So we will simply restate the bounds for the other terms and give more details for our modified $I_2$,


%
\begin{align}
I_1 &\leq \frac{r_{max}^2}{\gamma_T}, \label{proof_I_1}
\end{align}
where we use the fact that $(\eta_{t}-\eta_{t}^*)^2\leq 2r_{max}^2$.

\textbf{Bound on $I_2$:} 
\cite{suttle2023beyond} achieves this intermediate bound on the absolute value $I_2$.
\begin{align}
	|I_2|  
	&\leq \sum_{t=1}^{T}	\E\left[\big |(\eta_{t}-\eta_{t}^*)\big |\cdot \big |(F' (\eta_{t})-f_t^{j_{\max}})\big |\right] 
\end{align}
\cite{suttle2023beyond} proceeds to bound $(\eta_{t}-\eta_{t}^*)^2\leq 2r_{max}$. However, we will omit that step and bound in the term in the following way,
\begin{align}
	|I_2|  
	&\leq \sum_{t=1}^{T}	\E\left[\big |(\eta_{t}-\eta_{t}^*)\big |\cdot \big |(F' (\eta_{t})-f_t^{j_{\max}})\big |\right] 
    \\
     &\leq \sum_{t=1}^{T}	\E\big |(\eta_{t}-\eta_{t}^*)\big |\cdot \sum_{t=1}^{T}	\E\big |(F' (\eta_{t})-f_t^{j_{\max}})\big | . 
    \\
    &\leq \left(\sum_{t=1}^{T}	\E\left[\big |(\eta_{t}-\eta_{t}^*)\big |^2\right]\right)^{\frac{1}{2}}\left( \sum_{t=1}^{T}	\E\left[\big |(F' (\eta_{t})-f_t^{j_{\max}})\big |^2 \right] \right)^{\frac{1}{2}}\label{ref_9}
\end{align}
%
%
For $\left( \sum_{t=1}^{T}	\mathbb{E}\left[\big |(F' (\eta_{t})-f_t^{j_{\max}})\big |^2 \right]\right)^{\frac{1}{2}}$, we utilize Lemma \ref{lemma:a6} like \citet{suttle2023beyond} with $x_t = \eta_t, \nabla L(x_t) = \nabla F(\eta_t)$ and $l(x_t, z_t) = f_t$, and the fact that the Lipschitz constant of $\nabla F(\eta_t)$ is 1:
\begin{align}
	|I_2| &\leq \left(\sum_{t=1}^{T}	\E\left[\big |(\eta_{t}-\eta_{t}^*)\big |^2\right]\right)^{\frac{1}{2}} \wo{T\tau_{\text{mix}} \frac{\log T_{\max}}{T_{\max}}}^{\frac{1}{2}} . \label{proof_I_2}
\end{align}

\textbf{Bound on $I_3$:} 
%
%
\begin{equation}
|I_3| \leq \left( \sum_{t=1}^T \E \left[ (\eta_t - \eta_t^*)^2 \right] \right)^{1/2} \left( L^2 G_H^2 \sum_{t=1}^T \frac{\alpha_t^2}{\gamma_t^2} \right)^{1/2}. \label{ref_11}
\end{equation}

\textbf{Bound on $I_4$:} 
\begin{equation}
I_4  \leq  L^2 G_H^2 \sum_{t=1}^{T}\frac{\alpha^2}{\gamma_t}.
\end{equation}

\textbf{Bound on $I_5$:} 
\begin{equation}
I_5  \leq \sum_{t=1}^{T}\gamma_t \wo{ R^2 \tau_{\text{mix}}^{\theta_t} \log T_{\max} }. \label{proof_I_5}
\end{equation}

Combining the foregoing and recalling that $\gamma_t = (1+t)^{-\nu}, \alpha_t' = (1+t)^{-\sigma}$, $0 < \nu < \sigma < 1$, and $\alpha_t \leq \alpha_t'$, we get
\begin{align}
	\sum_{t=1}^{T} \E[(\eta_{t}-\eta_{t}^*)^2] 	
    %
    %
	&\leq  {2r_{\max}^2(1+T)^\nu} + \left[L^2G_H^2 + \wo{ \tau_{\text{mix}}\log T_{\max} } \right]\sum_{t=1}^{T}(1+t)^{-\nu} \\
	& \qquad + \left(\sum_{t=1}^{T}	\E\left[\big |(\eta_{t}-\eta_{t}^*)\big |^2\right]\right)^{\frac{1}{2}} \wo{T\tau_{\text{mix}} \frac{\log T_{\max}}{T_{\max}}}^{\frac{1}{2}}  \\
	& \qquad + \left(\sum_{t=1}^{T}\E[(\eta_t-\eta_t^*)^2]\right)^{\frac{1}{2}}\left( L^2 G_H^2 \sum_{t=1}^{T}(1+t)^{-2(\sigma-\nu)}\right)^{\frac{1}{2}},
\end{align}
where the second inequality follows from the fact that $\nu - 2\sigma < -\nu$.

Define
\begin{align}
    Z(T) &= \sum_{t=1}^{T} \E[(\eta_{t}-\eta_{t}^*)^2], \\
    F(T) &= \frac{L^2 G_H^2}{4} \sum_{t=1}^{T}(1+t)^{-2(\sigma-\nu)}, \\
    G(T) &=  T \wo{ \tau_{\text{mix}}\frac{\log T_{\max}}{T_{\max}} } \\
    A(T) &= 2r_{\max}^2(1+T)^\nu + \left[L^2G_H^2 + \wo{ \tau_{\text{mix}}\log T_{\max} } \right]\sum_{t=1}^{T}(1+t)^{-\nu} \\ \
\end{align}
The inequality can now be written as,
    \begin{align}
        Z(T) \leq A(T) + 2 \sqrt{ Z(T) } \sqrt{ F(T) } + 2 \sqrt{ Z(T) } \sqrt{ G(T) } \leq 2 A(T) + 16 F(T) + 16 G(T),
    \end{align}
By following the same steps as the critic error analysis in \cite{suttle2023beyond} to rearrange the above inequality, we achieve,
    \begin{align}
        Z(T) \leq 2 A(T) + 16 F(T) + 16 G(T).
    \end{align}
    %
    From $2 A(T) + 16 F(T) = \bo{T^{\nu}} + \bo{ T^{1 + \nu - 2\sigma}} + \bo{T^{1-\nu}}$ and using the bound $\sum_{t=1}^T (1+t)^{-\xi} \leq (1 + t)^{1-\xi} / (1 - \xi)$, we have by dividing by $T$,
    \begin{align}
        \frac{1}{T} \sum_{t=1}^T \E \left[ (\eta_t - \eta_t^*)^2 \right] \leq \bo{T^{\nu - 1}} + \bo{T^{-2(\sigma - \nu)}} 
    + \wo{  \tau_{mix} \log T_{max}} \bo{T^{-\nu}} 
    + \wo{  \tau_{mix} \frac{\log T_{max}}{T_{max} } }.
    \end{align}
    
    %
    %

\end{proof}

\section{Parameterized Policy Gradient with Advantage Estimation}
We repeat the algorithm for PPGAE as it appears in \cite{bai2023regret}.
\begin{algorithm}
    \caption{Parameterized Policy Gradient}
    \label{alg:PPGAE}
    \begin{algorithmic}[1]
        \STATE \textbf{Input:} Initial parameter $\theta_1$, learning rate $\alpha$,  initial state $s_0 \sim \rho(\cdot)$, episode length $H$ \vspace{0.1cm}
        \STATE $K=T/H$
	\FOR{$k\in\{1, \cdots, K\}$}
            \STATE $\mathcal{T}_k\gets \phi$
            
            \FOR{$t\in\{(k-1)H, \cdots, kH-1\}$}
                \STATE Execute $a_t\sim \pi_{\theta_k}(\cdot|s_t)$, receive reward $r(s_t,a_t) $ and observe $s_{t+1}$
                \STATE $\mathcal{T}_k\gets \mathcal{T}_k\cup \{(s_t, a_t)\}$
            \ENDFOR	
            
            \FOR{$t\in\{(k-1)H, \cdots, kH-1\}$}
                \STATE Using Algorithm \ref{alg:estQ}, and $\mathcal{T}_k$, compute $\hat{A}^{\pi_{\theta_k}}(s_t, a_t)$
            \ENDFOR
            \vspace{0.1cm}
      
            \STATE Compute  $\omega_k = \dfrac{1}{H}\sum_{t=t_k}^{t_{k+1}-1}\hat{A}^{\pi_{\theta}}(s_{t}, a_{t})\nabla_{\theta}\log \pi_{\theta_k}(a_{t}|s_{t})$ 
		\STATE Update parameters as
		\begin{equation}
                \label{udpates_algorotihm}
		    \theta_{k+1}=\theta_k+\alpha\omega_k
		\end{equation}
        \ENDFOR
    \end{algorithmic}
\end{algorithm}
\begin{algorithm}
    \caption{Advantage Estimation}
    \label{alg:estQ}
    \begin{algorithmic}[1]
        \STATE \textbf{Input:} Trajectory $(s_{t_1}, a_{t_1},\ldots, s_{t_2}, a_{t_2})$, state $s$, action $a$, and policy parameter $\theta$
        \STATE \textbf{Initialize:} $i \leftarrow 0$, $\tau\leftarrow t_1$
	\STATE \textbf{Define:} $N=4t_{\mathrm{mix}}\log_2T$.
	\vspace{0.1cm}
	\WHILE{$\tau\leq t_2-N$}
		\IF{$s_{\tau}=s$}
			\STATE $i\leftarrow i+1$.
			\STATE $\tau_i\gets \tau$
			\STATE $y_i=\sum_{t=\tau}^{\tau+N-1}r(s_t, a_t)$.
			\STATE $\tau\leftarrow\tau+2N$.	
		\ELSE
                \STATE {$\tau\leftarrow\tau+1$.}
            \ENDIF
	\ENDWHILE
        \vspace{0.1cm}
        \IF{$i>0$}
            \STATE $\hat{V}(s)=\dfrac{1}{i}\sum_{j=1}^i y_j$,
            \STATE $\hat{Q}(s, a) = \dfrac{1}{\pi_{\theta}(a|s)}\left[\dfrac{1}{i}\sum_{j=1}^i y_j\mathrm{1}(a_{\tau_j}=a)\right]$
        \ELSE
            \STATE $\hat{V}(s)=0$, $\hat{Q}(s, a) = 0$
        \ENDIF
	\STATE \textbf{return}  $\hat{Q}(s, a)-\hat{V}(s)$ 
    \end{algorithmic}
\end{algorithm}
\newpage
\section{Multi-level Actor-Critic}
We repeat the algorithm overview for MAC as it appears in \cite{suttle2023beyond}.
\begin{algorithm}
	\caption{\textbf{M}ulti-level Monte Carlo \textbf{A}ctor-\textbf{C}ritic (MAC)}
	\label{alg:PG_MAG}
	\begin{algorithmic}[1]
		\STATE \textbf{Initialize:} Policy parameter $\theta_0$, actor step size $\alpha_t$, critic step size $\beta_t$, average reward tracking step size $\gamma_t$, initial state $s_1^{(0)} \sim \mu_{0}(\cdot)$, maximum trajectory length $T_{\max}$.
		\FOR{$t=0$ {\bfseries to} $T-1$}
		\STATE Sample level length $j_t \sim \text{Geom}(1/2)$
		\FOR{$i = 1, \dots, 2^{j_t}$}
		\STATE Take action $a^i_t \sim \pi_{\theta_t}(\cdot | s^i_{t})$
		\STATE Collect next state $s^{i+1}_{t} \sim P(\cdot | s_t^{i}, a^i_{t})$
		\STATE Receive reward $r_t^{i} = r(s^i_t, a^i_t) $
		
		\ENDFOR
		\STATE Evaluate MLMC gradient $ f_t^{MLMC}$, $ 
 h_t^{MLMC}$, and $ g_t^{MLMC}$ via \eqref{eq:MLMC_Gradient}
		\STATE Update parameters following \eqref{Critic_update_0} 
		\ENDFOR
	\end{algorithmic}
\end{algorithm}

\end{document}